\documentclass{article}
\usepackage{fullpage}
\usepackage[title]{appendix}
\usepackage[utf8]{inputenc} 
\usepackage[T1]{fontenc}    
\usepackage{hyperref}       
\usepackage{url}            
\usepackage{geometry}
\usepackage[table]{xcolor}

\usepackage{algorithm}
\usepackage{algorithmic}
\usepackage{diagbox}
\usepackage{multirow}
\usepackage{booktabs}
\usepackage{amsfonts,amssymb,amsthm,amsmath,mathtools}
\usepackage{nicefrac}
\usepackage{microtype}
\usepackage{hhline}
\usepackage{subfigure}
\usepackage{indentfirst}
\usepackage{enumitem}
\usepackage{graphicx}
\usepackage{cleveref}
\usepackage{colortbl}
\usepackage{makecell}
\usepackage{wrapfig}
\usepackage{adjustbox}
\usepackage{array}
\usepackage{caption}
\usepackage{subcaption}
\usepackage{arydshln}
\usepackage{utfsym}

\newtheorem{theorem}{Theorem}
\newtheorem{lemma}{Lemma}
\newtheorem{corollary}{Corollary}
\newtheorem{proposition}{Proposition}

\usepackage{amsmath,amsfonts,bm}

\def\eqref#1{equation~\ref{#1}}

\def\1{\bm{1}}
\def\0{\bm{0}}

\DeclareMathAlphabet{\mathsfit}{\encodingdefault}{\sfdefault}{m}{sl}
\SetMathAlphabet{\mathsfit}{bold}{\encodingdefault}{\sfdefault}{bx}{n}

\newcommand{\citet}{\cite}
\newcommand{\citep}{\cite}

\title{\textbf{Do Large Language Models Capture the Diversity\\ in their Training Data?}}

\date{}

\author{
Youqi Wu\thanks{Department of Computer Science and Engineering, The Chinese University of Hong Kong, yqwu24@cse.cuhk.edu.hk} ,
Farzan Farnia\thanks{Department of Computer Science and Engineering, The Chinese University of Hong Kong, farnia@cse.cuhk.edu.hk}
}

\begin{document}
\maketitle

\begin{abstract}
    Large language models are trained to model conditional distributions over text, yet it remains inadequately understood whether they capture the full diversity of plausible outputs present in their training data. We study this question through an information-theoretic lens by comparing the \emph{conditional entropy} of model-generated outputs with that of the corresponding training data. Given paired input-output samples, we use conditional entropy and its matrix-based analogue based on \emph{von Neumann entropy} to measure output variability beyond what is explained by the conditioning input, without requiring multiple reference outputs for the same prompt. Across LLM families with publicly available training data, including OLMo, Pythia, and GPT-Neo, we consistently find that model-generated outputs exhibit lower conditional entropy than their training data, across different model scales, sequence lengths, and decoding strategies. We observe a similar conditional diversity gap beyond language modeling, including class-conditioned ImageNet generators and text-conditioned models trained on MS-COCO. To address this gap, we propose a post-hoc correction mechanism that generates multiple outputs for each input and reweights them through a matrix-entropy projection, increasing conditional diversity while remaining close to the original model distribution. We prove the concavity of the matrix-based conditional entropy functional, which makes the resulting entropy-constrained projection a convex optimization problem, and develop a scalable mirror-descent algorithm for its implementation. Our results reveal a systematic conditional diversity gap between modern generative models and their training data, and provide an information-theoretic framework for measuring and mitigating this gap.
\end{abstract}

\section{Introduction}
\label{sec:intro}
Large language models are increasingly used as conditional distribution models: given a prompt, instruction, context, or document, they are expected to produce not only a single plausible completion, but samples from a rich set of valid outputs. The same viewpoint underlies many other conditional generative systems, including text-to-image models, image captioning systems, and class-conditioned image generators~\citep{radford2018improving,brown2020language,rombach2022high,li2023blip2}. In these settings, the conditioning variable rarely determines a unique output. A prompt may admit multiple correct phrasings, reasoning paths, levels of detail, or stylistic choices; an image may admit several faithful captions; and a class label may correspond to a broad range of visual instances. A central question is therefore whether modern conditional generative models capture the full range of outputs supported by the data distribution.

Most existing evaluations do not directly answer this question. Standard language-modeling metrics such as likelihood or perplexity measure average predictive fit, while many generation benchmarks focus on accuracy, alignment, and semantic consistency. For image and multimodal generation, metrics such as CLIPScore~\citep{hessel2021clipscore}, MAUVE~\citep{pillutla2021mauve}, and conditional Fr\'echet-type distances~\citep{koo2025cfred} provide valuable tools for comparing generated and reference samples. However, a model could still perform well under such metrics while still concentrating its probability mass on a narrower subset of valid outputs. This is important for LLMs and prompt-guided generation, where the goal is not merely to produce one acceptable answer, but to represent the conditional variability present in natural text data.

A direct measurement of \emph{conditional} output variability is statistically challenging. In many datasets, we observe paired samples $(X,Y)$, where $X$ is the conditioned variable and $Y$ is the corresponding output, but we do not observe many independent human outputs for the exact same condition. Thus, estimating conditional entropy separately for each fixed prompt is usually infeasible. To address this, we study conditional entropy through paired input-output samples. At the classical level, Shannon conditional entropy measures the uncertainty in $Y$ that remains after observing $X$. At the kernel level, we use a matrix-based analog built from a product kernel on pairs $(X,Y)$ and subtract the entropy of the conditioning variable. This gives a way to measure output variability beyond the variability already explained by the inputs, without requiring repeated reference outputs for every condition.

Using this viewpoint, we find a consistent conditional output-range gap in modern generative models. Across LLM families with publicly accessible training data, including OLMo~\cite{groeneveld-etal-2024-olmo}, Pythia~\cite{pmlr-v202-biderman23a}, and GPT-Neo~\cite{black_gpt_neo_2021}, we empirically evaluate that model-generated continuations exhibit lower conditional entropy than the corresponding training data. We observe the same phenomenon beyond text, including class-conditioned ImageNet generation and text-conditioned image models trained on MS-COCO. These results suggest that the issue is not specific to one architecture or training approach. Rather, conditional generative models appear to represent a narrowed version of the conditional distribution present in their data, as their conditional entropy underestimates that of the training data.

We formalize this phenomenon using conditional entropy and matrix-based conditional entropy, which extends the recent diversity bias analysis in \cite{farnia2026exposing} from unconditional generative models to conditional generative AI and LLM models. The resulting score compares the entropy of the joint input-output distribution with the entropy of the input marginal, thereby isolating the part of the variability attributable to the conditional output law. This formulation will be natural for LLMs and general conditional generation systems, because it evaluates paired samples directly, rather than requiring multiple ground-truth responses for each prompt. It also connects the empirical measurement problem to matrix-based quantum entropy functionals.

On the technical side, we prove that the proposed matrix-based conditional entropy is a concave functional of the joint input-output distribution. This structural property has two consequences. It yields finite-sample monotonicity and underestimation properties analogous to the classical behavior of the classical Shannon entropy. More importantly, the concavity of the matrix-based quantum entropy implies that projecting a generated conditional distribution toward a higher conditional-entropy region can be formulated as a convex optimization problem.

Following the analysis, we propose a post-hoc sample-space reweighting method for LLMs and conditional generators. Given prompts $x_1,\ldots,x_N$, we generate and consider multiple outputs $y_{i,1},\ldots,y_{i,m}$ for each prompt $x_i$ and jointly optimize weights over these generated candidates so as to improve the overall conditional variability. The key constraint is that the total weight assigned to each prompt remains fixed, so the input marginal is preserved and only the conditional distribution over generated outputs is adjusted. The projection then finds the closest reweighted empirical distribution whose matrix-based conditional entropy exceeds a target level. In this way, the method does not retrain the model or synthesize new samples; instead, it asks whether the model has already generated valid but underweighted outputs, and redistributes probability mass toward a broader conditional empirical distribution.

For scalability, we develop an efficient mirror-descent implementation over a product of simplices, one simplex for the candidates generated from each condition. This extends exponentiated-gradient reweighting from the unconditional setting to conditional generation while exactly preserving the prompt marginal. Since the joint input-output feature space is a tensor-product space and can be prohibitively high-dimensional, we implement the method using sketched joint features. For finite-dimensional embeddings, we use random projections of Kronecker-product features; for shift-invariant kernels, we use direct joint random Fourier features for the product kernel. These approximations avoid explicitly forming the full tensor-product representation and reduce the cost of each optimization step to depend on the sketch dimension rather than the raw product dimension. The contributions of this work can be summarized as:

\begin{itemize}[leftmargin=*]
\item We formulate the problem of whether LLMs and conditional generative models capture the full range of valid training outputs using conditional entropy and its matrix-based counterpart.
\item We empirically demonstrate a consistent conditional output-range gap across LLMs with publicly available training data, class-conditioned ImageNet and MS-COCO generators.
\item We derive concavity and finite-sample properties of the applied matrix-based conditional entropy functional, and show that entropy-constrained reweighting can be formulated as a convex program.
\item We develop a scalable post-hoc reweighting algorithm based on product-simplex mirror descent and sketched joint features for adjusting conditional output distributions.
\end{itemize}


\section{Preliminaries}
\label{sec:preliminaries}
We consider a conditional generation problem with conditioning variable $X\in\mathcal X$ and output variable $Y\in\mathcal Y$. The data distribution is denoted by $P_{XY}$, with marginal $P_X$ and conditional law $P_{Y\,\vert\, X}$. A conditional generative model specifies a conditional distribution $Q_{Y\,\vert\, X}$, which together with the same input marginal induces the joint model distribution $Q_{XY}=P_XQ_{Y\,\vert\, X}$.

Let $k_X:\mathcal X\times\mathcal X\to\mathbb R$ and $k_Y:\mathcal Y\times\mathcal Y\to\mathbb R$ be positive semidefinite kernels. Throughout the paper, we use normalized kernels which satisfy $k_X(x,x)=1$ and $k_Y(y,y)=1$ for every $x \in \mathcal X$ and $y\in\mathcal Y$.\footnote{Any kernel function $k$ can be normalized by replacing $k(z,z')$ with $\widetilde k(z,z') = k(z,z')/\sqrt{k(z,z)k(z',z')}$.} For paired samples, we use the product kernel
\[
k_{XY}\big([x,y],[x',y']\big)=k_X(x,x')\cdot k_Y(y,y').
\]
Given samples $(x_1,y_1),\ldots,(x_n,y_n)$, let $K_X$ and $K_Y$ be the corresponding Gram matrices. The product kernel has Gram matrix where $\odot$ denotes the Hadamard product:
$K_{XY}=K_X\odot K_Y$

For a positive~semidefinite matrix $A$ with unit-trace $\operatorname{Tr}(A)=1$, its von~Neumann~entropy~is~defined~as
$$ H_{\mathrm{vN}}(A)=-\operatorname{Tr}\bigl(A\log A\bigr)$$
Following the matrix-based entropy framework of \cite{sanchezgiraldo2015measures,jalali2024conditional}, we measure conditional output uncertainty by the difference between the joint input--output entropy and the input entropy:
\[
H_{\mathrm{vN}}\bigl(Y\,\vert\, X;\widehat P_n\bigr)
:=
H_{\mathrm{vN}}\bigl(\frac1n K_{XY}\bigr)
-
H_{\mathrm{vN}}\bigl(\frac1n K_X\bigr).
\]
We note that the exponential of the above function is the same as the \emph{Conditional Vendi} score, which is introduced and analyzed by Jalali et al.~\cite{jalali2024conditional}. This quantity is the kernel analogue of conditional entropy: it measures the uncertainty contributed by the output after accounting for the entropy already present in the conditioning variable. 

We also use the order-$2$ matrix entropy
$H_2(A)=-\log \operatorname{Tr}(A^2)$,
which gives the empirical conditional order-$2$ entropy as follows where $\Vert \cdot \Vert_F$ denotes the Frobenius norm:
\[
H_2\bigl(Y\,\vert\, X;\widehat P_n\bigr)
:=
H_2\bigl(\frac1n K_{XY}\bigr)
-
H_2\bigl(\frac1n K_X\bigr)
=
\log\frac{\|K_X\|_F^2}{\|K_X\odot K_Y\|_F^2}.
\]
This order-$2$ version is useful as a computationally simpler counterpart to von Neumann entropy, as formulated and framed as the \emph{Conditional RKE} score in \citep{jalali2024conditional}. Population-level operator definitions and additional details are deferred to Appendix~\ref{app:preliminaries_extended}.


\section{Conditional output-range gaps in open language models}
\label{sec:llm_empirical_gap}

We first evaluate whether open language models match the conditional output range of their training data. We focus on OLMo, Pythia, and GPT-Neo, for which the training corpora are publicly available or reconstructable. Each example is represented as a paired sample $(X,Y)$, where $X$ is a prefix with $l_p \in \{3,4,5\}$ tokens and $Y$ is the corresponding continuation with $l_c \in \{2,3,4\}$ tokens. For the training baseline, $Y$ is the continuation observed in the corpus; for model samples, $Y$ is generated from the same prefix using greedy decoding, nucleus sampling, or ancestral sampling. We then compare the paired training and model distributions using conditional von Neumann entropy (VNE), conditional order-$2$ matrix entropy, and lexical Distinct-$2$ and Distinct-$3$ scores. 

In all numerical experiments and tables, we report the \textbf{exponentiated forms} $\exp(H_{\mathrm{vN}})$ and $\exp(H_2)$, so that the reported values can be interpreted as effective numbers of conditionally distinguishable outputs under the corresponding entropy measures. 
Table~\ref{tab:diversity_metrics_gap_20k} reports the results. The gap is defined as
$\Delta
=
\mathrm{Metric}_{\mathrm{train}}
-
\mathrm{Metric}_{\mathrm{model}}$;
therefore, positive values indicate that model-generated continuations have a smaller measured conditional output range than the training data. The pattern is consistent across all three model families and all decoding methods: the reported gap is positive. 

The lexical Distinct-$n$ scores provide a complementary surface-level check. They also show positive gaps in every setting, with the largest deficits again occurring under greedy decoding. However, Distinct-$n$ does not explicitly account for the conditioning prefix and only measures n-gram variation in the generated continuations. We therefore use it as supporting evidence, while the main comparison is based on conditional entropy over paired prefix--continuation samples. The confidence intervals, computed over five independent evaluations, are small relative to the observed gaps, indicating that the effect is stable rather than a fluctuation of a single sample draw.

These experiments motivate two questions. First, what structural property makes conditional von Neumann entropy a principled analogue of Shannon conditional entropy? Second, if a model has already generated multiple candidate outputs for each input, can we reweight those candidates to obtain a higher-entropy conditional empirical distribution while remaining close to the original model samples? The next section answers the first question through a concavity result. The following section uses this result to formulate post-hoc conditional reweighting as a convex projection problem.

\begin{table}
\centering
\caption{
Diversity metrics across models and generation strategies with sample size 20K.
For each model, we report conditional VNE, conditional RKE, and Distinct-2 as mean $\pm$ standard deviation.}
\label{tab:diversity_metrics_gap_20k}
\resizebox{\textwidth}{!}{
\begin{tabular}{llcccccc}
\toprule
\textbf{Model} & \textbf{Group}
& \textbf{Exp.Cond.VNE} & $\boldsymbol{\Delta}$\textbf{VNE}
& \textbf{Exp.Cond.RKE} & $\boldsymbol{\Delta}$\textbf{RKE}
& \textbf{Distinct-2} & $\boldsymbol{\Delta}$\textbf{D-2} \\
\midrule

\multirow{4}{*}{\rotatebox[origin=c]{90}{OLMo}}
& Training Data
& $338.58 {\scriptsize \pm 0.69}$ & --
& $67.43 {\scriptsize \pm 0.63}$ & --
& $0.709 {\scriptsize \pm 0.004}$ & -- \\
& Greedy Decoding
& $218.88 {\scriptsize \pm 0.34}$ & 119.70
& $46.17 {\scriptsize \pm 0.19}$ & 21.26
& $0.276 {\scriptsize \pm 0.001}$ & 0.433 \\
& Nucleus Sampling
& $287.31 {\scriptsize \pm 0.98}$ & 51.27
& $55.08 {\scriptsize \pm 0.27}$ & 12.35
& $0.516 {\scriptsize \pm 0.003}$ & 0.193 \\
& Ancestral Sampling
& $297.31 {\scriptsize \pm 0.78}$ & 41.27
& $56.81 {\scriptsize \pm 0.22}$ & 10.62
& $0.558 {\scriptsize \pm 0.003}$ & 0.151 \\

\midrule

\multirow{4}{*}{\rotatebox[origin=c]{90}{Pythia}}
& Training Data
& $262.55 {\scriptsize \pm 0.47}$ & --
& $53.44 {\scriptsize \pm 0.38}$ & --
& $0.701 {\scriptsize \pm 0.006}$ & -- \\
& Greedy Decoding
& $176.39 {\scriptsize \pm 0.27}$ & 86.16
& $40.98 {\scriptsize \pm 0.64}$ & 12.46
& $0.266 {\scriptsize \pm 0.001}$ & 0.435 \\
& Nucleus Sampling
& $232.24 {\scriptsize \pm 0.78}$ & 30.31
& $47.51 {\scriptsize \pm 0.27}$ & 5.93
& $0.525 {\scriptsize \pm 0.005}$ & 0.176 \\
& Ancestral Sampling
& $241.79 {\scriptsize \pm 0.23}$ & 20.76
& $49.00 {\scriptsize \pm 0.42}$ & 4.44
& $0.564 {\scriptsize \pm 0.006}$ & 0.137 \\

\midrule

\multirow{4}{*}{\rotatebox[origin=c]{90}{GPT-Neo}}
& Training Data
& $260.01 {\scriptsize \pm 0.52}$ & --
& $52.98 {\scriptsize \pm 0.23}$ & --
& $0.701 {\scriptsize \pm 0.008}$ & -- \\
& Greedy Decoding
& $179.36 {\scriptsize \pm 0.88}$ & 80.65
& $43.34 {\scriptsize \pm 0.93}$ & 9.64
& $0.284 {\scriptsize \pm 0.003}$ & 0.417 \\
& Nucleus Sampling
& $229.62 {\scriptsize \pm 0.42}$ & 30.39
& $48.48 {\scriptsize \pm 0.67}$ & 4.50
& $0.526 {\scriptsize \pm 0.002}$ & 0.175 \\
& Ancestral Sampling
& $238.50 {\scriptsize \pm 0.61}$ & 21.51
& $49.59 {\scriptsize \pm 0.32}$ & 3.39
& $0.570 {\scriptsize \pm 0.002}$ & 0.131 \\

\bottomrule
\end{tabular}
}
\end{table}

\section{Theoretical Features of Kernel-induced Conditional von Neumann Entropy}
\label{sec:conditional_vne_concavity}

We next establish the structural property needed for our proposed optimization of conditional von Neumann entropy. For Shannon entropy, we note that conditional entropy is well-known to be concave in the joint distribution $P_{XY}$. In this section, we prove that the kernel matrix-based (quantum) conditional entropy defined in \citep{sanchezgiraldo2015measures} used in our formulation is also a concave functional of the joint distribution. This property gives a convex-analytic justification for the projection method developed in Section~\ref{sec:conditional_reweighting}. We begin by reviewing the corresponding known fact for the classical Shannon entropy.

\begin{proposition}[Concavity of classical conditional Shannon entropy]
\label{prop:shannon_cond_concave_main}
For finite alphabets $\mathcal X$ and $\mathcal Y$, the Shannon conditional entropy $P_{XY}\longmapsto H(Y\vert X)$ 
is concave in the joint distribution $P_{XY}$.
\end{proposition}
\begin{proof}
This proposition is known in the literature. For completeness, we give a proof in the Appendix. 
\end{proof}

Next, we highlight that the concavity of the kernel-induced matrix-based (quantum) conditional entropy in \cite{sanchezgiraldo2015measures} does not follow from this result or a similar proof. This is because unlike the known identity for classical Shannon entropy that $H(Y\vert X)=\mathbb E_{x\sim P_X}[H(Y\vert X=x)]$,
the kernel-induced conditional von Neumann entropy in \citep{sanchezgiraldo2015measures} does not satisfy the same identity. In the proof of the following proposition, we utilize the strong subadditivity of the quantum entropy to prove the concavity of the quantum conditional entropy functional as defined in \citep{sanchezgiraldo2015measures}:

\begin{proposition}[Concavity of kernel-induced conditional von Neumann entropy]
\label{thm:conditional_vne_concave_main}
Let $k_X$ and $k_Y$ be normalized positive semidefinite kernels. Then, the conditional von Neumann entropy defined in \cite{sanchezgiraldo2015measures}
\[
P_{XY}
\longmapsto
\mathcal H_{\mathrm{vN}}(Y\vert X;P_{XY})
\]
is a concave functional of the joint distribution $P_{XY}$.
\end{proposition}
\begin{proof}
We provide the proof in the Appendix.
\end{proof}

Our shown Proposition~\ref{thm:conditional_vne_concave_main} highlights the key structural property needed by the correction method we propose in the next section. Since superlevel sets of concave functions are convex sets, for every threshold $\rho$ the set
\[
\bigl\{
Q:\:
\mathcal H_{\mathrm{vN}}(Y\vert X;Q)\ge \rho
\bigr\}
\]
is a convex set. Therefore, if we search over empirical distributions supported on a fixed set of generated samples, the constraint that the reweighted distribution has conditional entropy at least $\rho$ is a convex constraint. As a result, the conditional von Neumann entropy is suitable not only as a diagnostic, but further as the basis for a tractable post-hoc reweighting procedure.



\section{Kernel-based Entropy-Constrained Projection for Conditional Generators}
\label{sec:conditional_reweighting}

The empirical gaps in Section~\ref{sec:llm_empirical_gap} motivate a post-hoc correction problem. Suppose a conditional generator has already produced multiple candidate outputs for each input. We ask whether one can reweight these candidates so that the resulting conditional distribution moves toward a higher conditional-entropy region while remaining close to the original model distribution. The concavity of conditional von Neumann entropy gives a projection framework for this question.

Fix an input marginal $P_X$. For a conditional law $R_{Y\vert X}$, define
\[
\mathcal H_{\mathrm{vN}}(R_{Y\vert X};P_X)
:=
\mathcal H_{\mathrm{vN}}(Y\vert X;P_XR_{Y\vert X}),
\]
where $P_XR_{Y\vert X}$ denotes the joint law induced by $P_X$ and $R_{Y\vert X}$. For a threshold $\rho$, let
\[
\mathcal C_\rho(P_X)
=
\bigl\{R_{Y\vert X}:\mathcal H_{\mathrm{vN}}(R_{Y\vert X};P_X)\ge \rho\bigr\}.
\]
By Proposition~\ref{thm:conditional_vne_concave_main}, $\mathcal C_\rho(P_X)$ is a convex set.

We compare conditional laws through divergences lifted to joint laws with the same input marginal. Given a Bregman divergence $D_\Phi$ on joint laws, define
\[
D_{\Phi\vert P_X}\bigl(R_{Y\vert X},Q_{Y\vert X}\bigr)
:=
D_\Phi\bigl(P_XR_{Y\vert X},P_XQ_{Y\vert X}\bigr).
\]
For example, when $D_\Phi$ is KL divergence, this gives the $P_X$-averaged conditional KL. When $D_\Phi$ is the squared Hilbertian distance between kernel mean embeddings of joint pairs $(X,Y)$, it gives the conditional MMD with the input marginal held fixed.

\begin{theorem}
\label{thm:conditional_entropy_projection_principle}
Let $\rho=\mathcal H_{\mathrm{vN}}(P_{Y\vert X};P_X)$, and let $Q^\star_{Y\vert X}$ be the $D_{\Phi\vert P_X}$-Bregman projection of $Q_{Y\vert X}$ onto $\mathcal C_\rho(P_X)$. Then, we have
\[
D_{\Phi\vert P_X}\bigl(P_{Y\vert X},Q^\star_{Y\vert X}\bigr)
\le
D_{\Phi\vert P_X}\bigl(P_{Y\vert X},Q_{Y\vert X}\bigr) \, - \, D_{\Phi\vert P_X}\bigl(Q^\star_{Y\vert X},Q_{Y\vert X}\bigr)
\]
\end{theorem}

The theorem gives a geometric interpretation of the projection step. If the entropy threshold is chosen so that the data conditional law is feasible, then the projection cannot increase the Bregman discrepancy from the data law. The second term is the Pythagorean decrease obtained by replacing $Q_{Y\vert X}$ with its high-entropy projection. The proof is provided in Appendix~\ref{app:proofs_additional_results}.

We now instantiate this principle on generated samples. Let $x_1,\ldots,x_N$ be input prompts, and let $y_{i,1},\ldots,y_{i,m}$ be $m$ generated outputs for prompt $x_i$. We assign weights $p_i\in\Delta_m$ to the candidates generated from $x_i$ and place joint mass $q_{i,j}=\tfrac{1}{N}p_{i,j}$ on $(x_i,y_{i,j})$. Hence each input keeps total mass $1/N$, and the reweighting modifies only the conditional distribution over generated outputs.

Let $K\in\mathbb R^{Nm\times Nm}$ be the product-kernel matrix over generated pairs,
\[
K_{(i,j),(i',j')}
=
k_X(x_i,x_{i'})k_Y(y_{i,j},y_{i',j'}).
\]
Let $u_{i,j}=1/(Nm)$ represent the uniform weight distribution (original baseline). We formulate the finite-support projection as
\begin{equation}
\label{eq:kernel_conditional_projection}
\begin{aligned}
\min_{q\in\mathbb{R}^{N m}}
\quad&
(q-u)^\top K(q-u)
\\
\mathrm{s.t.}
\quad&
\mathcal H_{\mathrm{vN}}(q)\ge \rho,
\\
&
q_{i,j}\ge 0,\qquad
\sum_{j'=1}^m q_{i,j'}=\frac1N
\quad \text{for all }\: i\in\{1,\ldots ,N \} ,\, j\in\{1,\ldots ,m \} .
\end{aligned}
\end{equation}
The objective is the finite-sample conditional MMD between the reweighted empirical distribution and the original uniformly weighted model samples. The entropy term is the empirical conditional von Neumann entropy of the weighted paired sample. Since the block constraints fix the empirical input marginal, the input-entropy term is constant over the feasible set.

\begin{proposition}
\label{prop:finite_support_projection}
The optimization problem~\eqref{eq:kernel_conditional_projection} is a convex program, and its objective is the finite-support instance of the lifted squared-MMD discrepancy $D_{\Phi\vert P_X}$.
\end{proposition}

For large sample sets, we solve a covariance-space approximation of~\eqref{eq:kernel_conditional_projection}. Let $\tilde z_{i,j}\in\mathbb R^r$ be normalized sketched features for the pair $(x_i,y_{i,j})$, chosen so that their inner products approximate the product kernel. Define
\[
\widetilde C(p)
=
\frac1N\sum_{i=1}^N\sum_{j=1}^m
p_{i,j}\tilde z_{i,j}\tilde z_{i,j}^{\top}.
\]
We optimize the penalized covariance-space objective
\begin{equation}
\label{eq:covariance_penalized_projection}
\min_{p_1,\ldots,p_N\in\Delta_m}
\quad
F(p):=
\Bigl\Vert 
\frac1N\sum_{i,j}(p_{i,j}-1/m)\tilde z_{i,j}
\Bigr\Vert_2^2
-
\lambda H_{\mathrm{vN}}(\widetilde C(p)).
\end{equation}

\begin{proposition}
\label{prop:covariance_space_projection}
The optimization problem~\eqref{eq:covariance_penalized_projection} is a convex program. It is the covariance-space analogue of the kernel projection~\eqref{eq:kernel_conditional_projection} under the sketched product-kernel representation.
\end{proposition}

We propose the algorithm \emph{Conditional Entropy Projection by Block Exponentiated Gradient} (CEP-BEG) for solving~\eqref{eq:covariance_penalized_projection}. As described in Algorithm~\ref{alg:cep_beg}, CEP-BEG performs block mirror descent over the product simplex $\Delta_m^N$, with separated exponentiated-gradient (EG) normalization per input block. Note that we include this block normalization to preserve the empirical input marginal. Theorem~\ref{thm:cep_beg_convergence_main} provides a convergence guarantee for the averaged iterate in CEP-BEG.

\begin{algorithm}[t]
\caption{CEP-BEG: Conditional Entropy Projection by Block Exponentiated Gradient}
\label{alg:cep_beg}
\begin{algorithmic}[1]
\REQUIRE Inputs $\{x_i\}_{i=1}^N$, generated outputs $\{y_{i,j}\}_{j=1}^m$, sketch dimension $r$, entropy weight $\lambda$, step size $\eta$, iterations $T$.
\STATE Construct normalized sketched joint features $\{\tilde z_{i,j}\in\mathbb R^r\}$.
\STATE Initialize $p_{i,j}^{(0)}=1/m$ for all $i,j$.
\FOR{$t=0,\ldots,T-1$}
\STATE Form the weighted covariance $\widetilde C(p^{(t)})$.
\STATE Compute a gradient or subgradient $g^{(t)}$ of $F$ at $p^{(t)}$.
\STATE Update each input block by
$
p_{i,j}^{(t+1)}
=
{
p_{i,j}^{(t)}\exp(-\eta g_{i,j}^{(t)})
}/\bigl({
\sum_{\ell=1}^m p_{i,\ell}^{(t)}\exp(-\eta g_{i,\ell}^{(t)})
}\bigr)$.
\ENDFOR
\RETURN Weights $q_{i,j}=p_{i,j}^{(T)}/N$ on the generated pairs.
\end{algorithmic}
\end{algorithm}

\begin{theorem}[Convergence of CEP-BEG]
\label{thm:cep_beg_convergence_main}
Assume $F$ is convex and that the CEP-BEG gradients satisfy
$\sum_{i=1}^N \|g_i^{(t)}\|_\infty^2\le G^2$ for all $t$ 
where $g_i^{(t)}=(g_{i,1}^{(t)},\ldots,g_{i,m}^{(t)})$. Let $\bar p^{(T)}=\tfrac{1}{T}\sum_{t=0}^{T-1}p^{(t)}$. Then, choosing $\eta=\sqrt{2N\log m}/(G\sqrt T)$, for any minimizer $p^\star$ of~\eqref{eq:covariance_penalized_projection} we have 
\[
F(\bar p^{(T)})-F(p^\star)
\le
G\sqrt{\frac{2N\log m}{T}}.
\]
\end{theorem}

The sketched features can be constructed in two complementary ways. For finite-dimensional embeddings, Gaussian random projection of the tensor-product feature gives a Johnson--Lindenstrauss-type guarantee on the generated sample set. For shift-invariant kernels, direct joint random Fourier features approximate the product kernel without forming tensor products. With sketch dimension $r$, each CEP-BEG iteration costs $O(Nmr^2+r^3)$ using a full eigendecomposition of the $r\times r$ covariance matrix. The proofs, gradient and feature constructions, and concentration bounds are in Appendix~\ref{app:proofs_additional_results}.

\section{Numerical Results}
\label{sec:numerical}

\subsection{Experimental Setup}

\textbf{Models and datasets.}
Section~\ref{sec:llm_empirical_gap} studies the conditional output-range gap for three fully open language models: OLMo~\cite{groeneveld-etal-2024-olmo}, Pythia~\cite{pmlr-v202-biderman23a}, and GPT-Neo~\cite{black_gpt_neo_2021}. Specifically, OLMo is trained on Dolma~\cite{soldaini2024dolma}, a large-scale open text corpus designed to support transparent language-model pretraining. Pythia and GPT-Neo are trained on The Pile~\cite{gao2020pile}, a diverse collection of English text from multiple domains, including web text, books, academic papers, code, and other curated sources.
Beyond language models, we evaluate two conditional image-generation settings. For class-conditioned ImageNet generation, we follow the DGM-Eval benchmark~\cite{stein2023exposing} and use generated samples from LDM~\cite{ldm_rombach2022high}, ADM~\cite{adm_dhariwal2021diffusion}, BigGAN~\cite{biggan_brock2019large}, and DiT-XL-2~\cite{dit_xl_2_peebles2023scalable}, with ImageNet~\cite{deng2009imagenet} training images as the real-data reference. For text-conditioned MS-COCO generation, we evaluate U-ViT~\cite{uvit_bao2023all}, SDXL~\cite{podell2024sdxl}, and PixArt~\cite{chen2024pixart}, using MS-COCO 2014~\cite{coco_lin2014microsoft} training set as the reference distribution.

\textbf{Tasks and evaluation protocol.}
We first evaluate conditional diversity gaps across language and image generation tasks, and then study two applications of the proposed conditional diversity objective.
For language modeling, following Section~\ref{sec:llm_empirical_gap}, we compare training-corpus continuations with model continuations generated from the same prefixes using greedy decoding, nucleus sampling with $p=0.9$, and ancestral sampling.
For image generation, we compare real and generated images under matched conditions, including class labels on ImageNet and captions on MS-COCO.
We then evaluate the proposed objective in two application settings.
For language models, we apply entropy-projected reweighting to multiple candidate continuations generated for each prefix, while preserving the prefix marginal.
For text-to-image generation, we incorporate the conditional diversity objective into the diffusion sampling process as a sampling-time guidance term, encouraging higher-diversity outputs under the same text condition on MS-COCO dataset.

\textbf{Feature representations and kernels.}
We extract features using four fixed pretrained models: Qwen3-Embedding~\cite{zhang2025qwen3embedding}, CLIP~\cite{clip_radford2021learning}, T5~\cite{t5_raffel2020exploring}, and DINOv2~\cite{oquab2024dinov2}. 
Then we construct kernels on the resulting features and evaluate three kernel families: Gaussian, cosine, and degree-3 polynomial kernels. All kernels are normalized before computing conditional VNE and conditional RKE.

\subsection{Conditional Diversity Gaps in Language Models}
\label{sec:numerical_llm_gap}

Figure~\ref{fig:llm_sample_size_gap} reports conditional VNE across sample sizes from $5$K to $20$K for OLMo, Pythia, and GPT-Neo. 
We also observe that conditional VNE increases with sample size for both training and generated continuations. However, the training-data curve grows faster than the generated-sample curves, so the measured diversity gap generally becomes larger as more samples are included. 
The gap is not restricted to the approximately $1$B-parameter models used in the main experiments: we additionally evaluated OLMo-3-7B, Pythia-2.8B, and Pythia-6.9B, and the conditional output-range gap remains substantial across all three larger models; the complete sample-size curves are reported in Appendix~\ref{app:larger_models}. We also perform a human-interpretable lexical audit counting distinct semantic categories in training versus generated continuations, which recovers the same ordering as conditional VNE; the details are reported in Appendix~\ref{app:lexical_audit}.

\subsection{Long-Sequence Conditional Diversity Gaps}
\label{sec:long_sequence_gap}

To test whether the gap persists beyond the short prefix--continuation lengths in Section~\ref{sec:llm_empirical_gap}, we extend the experiments to prefix lengths $\{8,16,32\}$ and continuation lengths $\{16,32,64\}$ tokens across GPT-Neo, OLMo, and Pythia, representing each sequence as a concatenation of $2$-token segment embeddings; the construction details and a chunk-size ablation are reported in Appendix~\ref{app:chunk_size_ablation}. Table~\ref{tab:long_sequence_gap} reports the exponentiated conditional VNE for the training reference and the three decoding strategies, along with the average gap. The training data exhibit higher conditional entropy than every decoding strategy in all evaluated configurations, indicating that the conditional diversity shortfall is a robust phenomenon across sequence lengths.

\begin{table}[t]
\centering
\caption{
Conditional diversity gaps for long prefix--continuation sequences.
We report the exponential of conditional VNE for the training reference and generated continuations under greedy decoding, nucleus sampling, and ancestral sampling, using prefix lengths $\{8,16,32\}$ and continuation lengths $\{16,32,64\}$.
}
\label{tab:long_sequence_gap}
\resizebox{\textwidth}{!}{
\begin{tabular}{llcccccc}
\toprule
\textbf{Model} & \textbf{Prefix} & \textbf{Continuation}
& \textbf{Training} & \textbf{Greedy} & \textbf{Nucleus} & \textbf{Ancestral} & \textbf{Avg. Gap} \\
\midrule

\multirow{9}{*}{\rotatebox[origin=c]{90}{GPT-Neo}}
& $8$  & $16$ & $456.9$ & $366.1$ & $414.5$ & $418.5$ & $57.2$ \\
& $8$  & $32$ & $511.4$ & $394.5$ & $452.5$ & $458.4$ & $76.3$ \\
& $8$  & $64$ & $558.7$ & $416.7$ & $485.1$ & $494.9$ & $93.1$ \\
& $16$ & $16$ & $482.0$ & $438.4$ & $456.9$ & $460.0$ & $30.2$ \\
& $16$ & $32$ & $520.4$ & $460.1$ & $483.0$ & $488.2$ & $43.3$ \\
& $16$ & $64$ & $558.8$ & $479.5$ & $504.9$ & $512.3$ & $59.9$ \\
& $32$ & $16$ & $513.8$ & $494.9$ & $502.6$ & $505.5$ & $12.8$ \\
& $32$ & $32$ & $536.6$ & $507.3$ & $517.6$ & $519.9$ & $21.7$ \\
& $32$ & $64$ & $563.9$ & $519.3$ & $530.6$ & $535.5$ & $35.4$ \\

\midrule

\multirow{9}{*}{\rotatebox[origin=c]{90}{OLMo}}
& $8$  & $16$ & $462.9$ & $410.3$ & $433.0$ & $439.4$ & $35.3$ \\
& $8$  & $32$ & $520.0$ & $456.5$ & $484.9$ & $489.0$ & $43.2$ \\
& $8$  & $64$ & $575.6$ & $502.1$ & $534.0$ & $541.2$ & $49.8$ \\
& $16$ & $16$ & $487.8$ & $456.9$ & $470.5$ & $474.0$ & $20.7$ \\
& $16$ & $32$ & $530.4$ & $488.0$ & $506.0$ & $510.4$ & $28.9$ \\
& $16$ & $64$ & $576.9$ & $521.7$ & $545.4$ & $549.2$ & $38.1$ \\
& $32$ & $16$ & $524.0$ & $506.1$ & $513.6$ & $516.0$ & $12.1$ \\
& $32$ & $32$ & $551.4$ & $524.7$ & $534.8$ & $538.1$ & $18.9$ \\
& $32$ & $64$ & $585.2$ & $546.9$ & $561.1$ & $564.2$ & $27.8$ \\

\midrule

\multirow{9}{*}{\rotatebox[origin=c]{90}{Pythia}}
& $8$  & $16$ & $463.3$ & $374.9$ & $423.0$ & $429.7$ & $54.1$ \\
& $8$  & $32$ & $519.2$ & $415.3$ & $468.2$ & $475.2$ & $66.3$ \\
& $8$  & $64$ & $568.2$ & $452.9$ & $506.7$ & $509.5$ & $78.5$ \\
& $16$ & $16$ & $488.6$ & $449.1$ & $465.9$ & $469.3$ & $27.2$ \\
& $16$ & $32$ & $527.7$ & $480.1$ & $496.6$ & $499.3$ & $35.7$ \\
& $16$ & $64$ & $567.7$ & $510.6$ & $519.8$ & $523.4$ & $49.8$ \\
& $32$ & $16$ & $521.6$ & $503.4$ & $511.5$ & $513.7$ & $12.1$ \\
& $32$ & $32$ & $545.0$ & $519.5$ & $528.3$ & $530.2$ & $19.0$ \\
& $32$ & $64$ & $572.8$ & $536.4$ & $539.8$ & $542.7$ & $33.2$ \\

\bottomrule
\end{tabular}
}
\end{table}

\subsection{Conditional Diversity Gaps Beyond Language Models}
\label{sec:gap_beyond_llms}

Figure~\ref{fig:imagenet_gap} and Figure~\ref{fig:coco_gap} report conditional VNE across sample sizes from $2.5$K to $20$K for class-conditioned ImageNet generation and text-conditioned MS-COCO generation, respectively. In both settings, the real-data reference consistently achieves higher conditional VNE than the generated samples under matched conditions. We also observe that the gap generally increases with sample size, as the real-data conditional VNE grows faster than the generated-sample conditional VNE. This trend is consistent with the LLM experiments in Section~\ref{sec:llm_empirical_gap}.

\begin{figure}

\begin{center}
\centerline{\includegraphics[width=\textwidth]{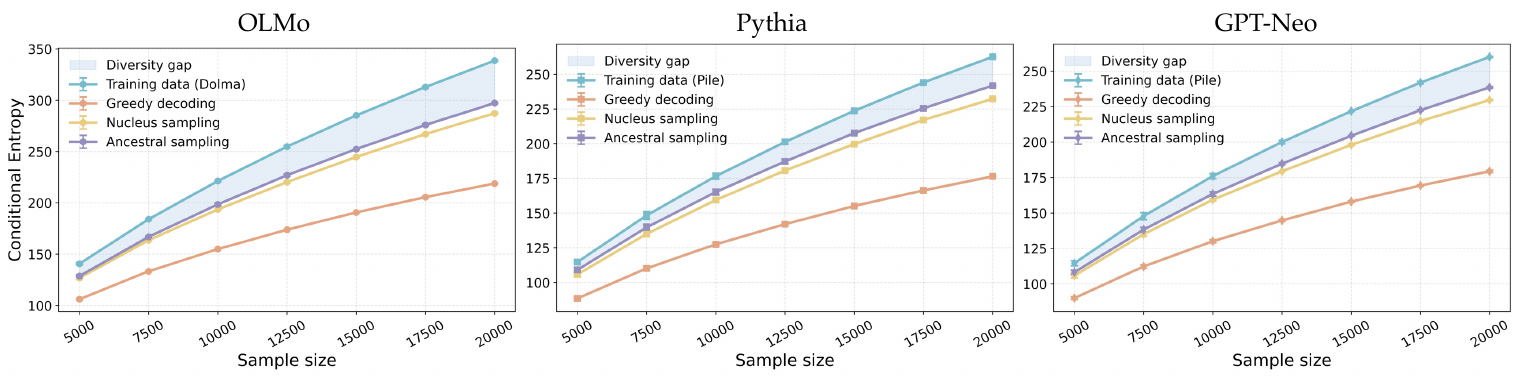}}
\caption{
Conditional diversity gaps between training data and generated data in language models.
}
\label{fig:llm_sample_size_gap}
\end{center}

\end{figure}


\subsection{Entropy-Projected Reweighting for Conditional Generation}
\label{sec:entropy_projected_reweighting}

We next evaluate the proposed entropy-projected reweighting method on language generation. For each prefix, we generate 10 candidate continuations and reweight them using the conditional entropy projection in Section~\ref{sec:conditional_reweighting}. 
Table~\ref{tab:projection_conditional_diversity} compares the original uniformly weighted candidates, denoted as ``w/o Projection'', with the reweighted candidates, denoted as ``w/ Projection''.
These results indicate that the proposed projection can increase the measured conditional diversity of generated samples without retraining the model or generating new outputs. 

\begin{table}[t]
\centering
\caption{
Effect of entropy-projected reweighting on conditional diversity.
We report exponential of conditional VNE and conditional RKE at sample sizes 10K and 20K, averaged over 5 runs.
}
\label{tab:projection_conditional_diversity}
\resizebox{\linewidth}{!}{
\begin{tabular}{llcccc}
\toprule
\textbf{Model} & \textbf{Setting}
& \multicolumn{2}{c}{\textbf{Sample Size = 10K}}
& \multicolumn{2}{c}{\textbf{Sample Size = 20K}} \\
\cmidrule(lr){3-4}
\cmidrule(lr){5-6}
& & \textbf{Exp.Cond.VNE} & \textbf{Exp.Cond.RKE}
  & \textbf{Exp.Cond.VNE} & \textbf{Exp.Cond.RKE} \\
\midrule

\multirow{2}{*}{OLMo}
& w/o Projection & 228.17\(\pm 2.54 \) & 222.55\(\pm 4.28\) & 366.35\(\pm 2.01\) & 241.14\(\pm 4.27\) \\
& w/ Projection  & 233.88\(\pm 1.89\) & 257.25\(\pm 6.37\) & 376.77\(\pm 1.90\) & 284.92\(\pm 5.34\) \\

\midrule
\multirow{2}{*}{Pythia}
& w/o Projection & 230.55\(\pm 1.46 \)& 217.22\(\pm 6.02 \) & 371.72\(\pm 1.63\) & 236.24\(\pm 3.98\) \\
& w/ Projection  & 236.94\(\pm 1.84\) & 254.40\(\pm 5.69\) & 382.74\(\pm 1.38\) & 285.05\(\pm 4.11\) \\

\midrule
\multirow{2}{*}{GPT-Neo}
& w/o Projection & 211.18\(\pm 1.92\) & 225.51\(\pm 4.85\) & 340.90\(\pm 1.48\) & 246.13\(\pm 4.26\) \\
& w/ Projection  & 219.08\(\pm 0.70\) & 268.23\(\pm 6.24\) & 353.68\(\pm 2.20\) & 296.02\(\pm 3.01\) \\

\bottomrule
\end{tabular}
}
\end{table}

We also compare the proposed projection with a temperature-scaling baseline and conduct a qualitative probe of the reweighting behavior; both are reported in Appendices~\ref{app:temperature_baseline} and~\ref{app:depeaking_probe}. The temperature comparison shows that matching the conditional entropy alone is insufficient, and the probe confirms that the entropy gain comes from de-peaking valid modes rather than promoting low-quality outputs.

\subsection{Conditional VNE Guidance for Diffusion Sampling}

We further evaluate whether the conditional VNE score can be used not only as an evaluation metric, but also as a sampling-time guidance signal for conditional generation. We apply conditional VNE guidance to SDXL on MS-COCO captions. Detailed implementation is provided in Appendix~\ref{app:conditional_vne_guidance}.
Table~\ref{tab:mscoco_sdxl_guidance_cvs} reports the results. SDXL with conditional VNE guidance achieves higher conditional VNE than SDXL without guidance across all evaluated sample sizes.

\begin{table}[t]
\centering
\caption{Text-conditioned image generation on MS-COCO using SDXL with and without the guidance.
We report exponential of conditional VNE as mean $\pm$ standard deviation over 5 independent runs.
}
\label{tab:mscoco_sdxl_guidance_cvs}
\resizebox{\textwidth}{!}{
\begin{tabular}{lccccc}
\toprule
\textbf{Data / Model}
& \textbf{10K}
& \textbf{12.5K}
& \textbf{15K}
& \textbf{17.5K}
& \textbf{20K} \\
\midrule
Reference MS-COCO
& $28.51 \pm 0.06$
& $31.93 \pm 0.06$
& $35.03 \pm 0.03$
& $37.92 \pm 0.04$
& $40.52 \pm 0.03$ \\
SDXL w/o Guidance
& $22.51 \pm 0.09$
& $24.76 \pm 0.07$
& $26.76 \pm 0.07$
& $28.58 \pm 0.04$
& $30.25 \pm 0.04$ \\
SDXL w/ Guidance
& $23.75 \pm 0.09$
& $26.21 \pm 0.06$
& $28.44 \pm 0.07$
& $30.43 \pm 0.04$
& $32.27 \pm 0.02$ \\
\bottomrule
\end{tabular}
}
\end{table}

\subsection{Downstream Application: Conditional VNE-Weighted MBR Decoding}
\label{sec:cvs_mbr}

We next evaluate the proposed reweighting framework on a downstream task, model-based minimum Bayes risk (MBR) decoding~\cite{jinnai2024mbmbmbr}, where replacing the original candidate weights with the conditional entropy projection yields a conditional-Vendi-weighted variant (CVS-MBR) that we compare with Monte Carlo MBR (MC-MBR), model-based MBR (MBMBR), and its length-normalized variant (MBMBR-L) on three summarization settings; implementation details are provided in Appendix~\ref{app:evaluation_hyperparams}. Table~\ref{tab:cvs_mbr} reports conditional VNE together with ROUGE-L, BERTScore, Distinct-$2$, and Self-BLEU-$2$. CVS-MBR achieves the highest conditional VNE and Distinct-$2$, as well as the lowest Self-BLEU-$2$, across all three settings, while a paired bootstrap shows no statistically significant degradation in ROUGE-L or BERTScore relative to MBMBR-L, demonstrating a diversity-oriented quality--diversity trade-off.

\begin{table}[t]
\centering
\caption{
Conditional VNE-weighted MBR decoding on three summarization settings.
We report conditional VNE, ROUGE-L, BERTScore, Distinct-$2$, and Self-BLEU-$2$.
Arrows indicate the desirable direction for each metric.
CVS-MBR denotes the proposed conditional-Vendi-weighted variant.
}
\label{tab:cvs_mbr}
\resizebox{\textwidth}{!}{
\begin{tabular}{llccccc}
\toprule
\textbf{Setting} & \textbf{Method}
& \textbf{Cond.VNE} $\uparrow$ & \textbf{ROUGE-L} $\uparrow$
& \textbf{BERTScore} $\uparrow$ & \textbf{Distinct-2} $\uparrow$ & \textbf{Self-BLEU-2} $\downarrow$ \\
\midrule

\multirow{4}{*}{\shortstack{XSum\\BART-large-XSum}}
& MC-MBR   & $27.88$ & $0.335$ & $0.915$ & $0.936$ & $0.209$ \\
& MBMBR    & $4.93$  & $0.346$ & $0.918$ & $0.929$ & $0.212$ \\
& MBMBR-L  & $26.44$ & $0.350$ & $0.917$ & $0.938$ & $0.207$ \\
& CVS-MBR  & $28.33$ & $0.331$ & $0.915$ & $0.945$ & $0.197$ \\
\midrule

\multirow{4}{*}{\shortstack{CNN/DailyMail\\BART-large-CNN}}
& MC-MBR   & $33.00$ & $0.268$ & $0.879$ & $0.944$ & $0.186$ \\
& MBMBR    & $3.35$  & $0.273$ & $0.880$ & $0.945$ & $0.184$ \\
& MBMBR-L  & $28.80$ & $0.266$ & $0.879$ & $0.948$ & $0.174$ \\
& CVS-MBR  & $33.34$ & $0.272$ & $0.879$ & $0.951$ & $0.165$ \\
\midrule

\multirow{4}{*}{\shortstack{CNN/DailyMail\\Mistral-7B-Instruct}}
& MC-MBR   & $12.20$ & $0.256$ & $0.880$ & $0.919$ & $0.233$ \\
& MBMBR    & $4.62$  & $0.254$ & $0.881$ & $0.924$ & $0.229$ \\
& MBMBR-L  & $13.15$ & $0.261$ & $0.881$ & $0.920$ & $0.234$ \\
& CVS-MBR  & $14.09$ & $0.261$ & $0.881$ & $0.926$ & $0.224$ \\
\bottomrule
\end{tabular}
}
\end{table}

\textbf{Ablation Studies.}
\label{sec:ablation_studies}
To examine the robustness of the conditional diversity gaps, we conduct ablation studies under different experimental choices. Specifically, we vary the evaluation sample size, the kernel function, and the embedding model. For the language-model experiments, we further vary the prefix--continuation construction by considering prefix token lengths in $\{3,4,5\}$ and continuation token lengths in $\{2,3,4\}$. For the long-sequence experiments, we additionally perform a chunk-size ablation over chunk sizes $\{1,2,4\}$ used in the segment concatenation representation. For the sample-size analysis, we draw subsamples from 2.5k to 20k with different random seeds out of 30k samples. Full results are provided in the Appendix.

\section{Conclusion and Limitations}
\label{sec:conclusion}
We studied whether LLMs and conditional generative models capture the full range of valid outputs present in their training distributions. Using conditional entropy and its matrix-based von Neumann analogue, we observed a consistent conditional output-range gap across open language models and conditional image-generation settings. We specifically showed the concavity of kernel-induced conditional von Neumann entropy, supporting both its use as a diagnostic and the convexity of entropy-constrained projection. Building on this structure, we proposed a post-hoc reweighting method that preserves the input marginal while redistributing mass across generated candidates, with an efficient product-simplex mirror-descent implementation using sketched joint features.

The proposed entropy-based view complements standard evaluations of factuality, alignment, and semantic correctness. Our direct training-data comparison focuses on models for which the relevant training data are publicly available or reconstructable; for many widely used LLMs and generative AI models, the training data are not publicly accessible, making the analysis more difficult but potentially valuable. These considerations suggest natural extensions to closed-source systems and alternative feature choices.

\bibliographystyle{unsrt}
\bibliography{ref}

\begin{appendices}
\section{Related Work}
\label{sec:related}
\textbf{Evaluation of conditional generative AI.}
Evaluation in conditional generation is shaped primarily by fidelity, alignment, and preference-based criteria rather than by intrinsic variety under fixed conditions. In text-to-image generation, representative examples include CLIPScore~\citep{hessel2021clipscore}, TIFA~\citep{hu2023tifa}, ImageReward~\citep{xu2023imagereward}, and more recent conditional fidelity metrics such as cFreD~\citep{koo2025cfred}. In open-ended text generation, MAUVE~\citep{pillutla2021mauve} has become a prominent distributional metric, while lexical diversity diagnostics such as distinct-$n$~\citep{li2016diversity} and Self-BLEU~\citep{zhu2018texygen} remain common. In image captioning, standard evaluation still centers on BLEU~\citep{papineni2002bleu}, CIDEr~\citep{vedantam2015cider}, and SPICE~\citep{anderson2016spice}. In text-to-video generation, common practice combines video-quality metrics such as FVD~\citep{unterthiner2019fvd} with newer alignment-oriented metrics such as T2VScore~\citep{wu2024t2vscore} and evaluation suites such as FETV~\citep{liu2023fetv}. Also, the Scendi score \citep{ospanov2025scendi} and Conditional Vendi/RKE scores \citep{jalali2024conditional,jalali2025sparke} provide conditional diversity scores for prompt-aware diversity quantification. These works have substantially advanced benchmarking for conditional generation, but they are not designed to isolate the variety that a model contributes under a fixed condition. Our work complements this literature by focusing specifically on conditional variety as a first-class evaluation target.

\textbf{Variety shortfall, novelty, and memorization in modern generative AI.}
A growing empirical literature suggests that diversity shortfall is a real phenomenon in modern generative systems. Specifically, Farnia et al.~\citep{farnia2026exposing} conduct a diversity gap analysis for unconditional generative models, numerically showing a von Neumann entropy gap between standard generative models and their training data, which we extend to the conditional generation case by analyzing the conditional von Neumann entropy function. In image generation, Dombrowski et al.~\citep{dombrowski2025diversity} report that current state-of-the-art models cover only a limited fraction of the diversity present in the training distribution. 
In LLMs, Yun et al.~\citep{yun2025price} identify \emph{diversity collapse}, where formatting and instruction structure drive semantically similar outputs even under decoding regimes intended to encourage variety. In image captioning, Wang and Chan~\citep{wang2019describing} show that models optimized for standard captioning metrics often produce generic captions and remain far below human performance in diversity. Related concerns arise in the memorization and novelty literature. Carlini et al.~\citep{carlini2021extracting} demonstrate that large language models can regurgitate memorized training data, while Carlini et al.~\citep{carlini2023extracting} show analogous extraction phenomena for diffusion models. More recent empirical studies also investigate memorization behavior in diffusion models directly~\citep{gu2024memorization}. Our work differs in that it uses the conditional kernel entropy to analyze a shortfall relative to the empirical training distribution.

\textbf{Entropy-based analysis of LLMs.}
Entropy has recently become a useful tool for studying uncertainty, reliability, and training dynamics in LLMs. Semantic entropy~\citep{kuhn2023semantic,farquhar2024detecting} estimates uncertainty over meanings rather than surface forms and has been used for hallucination detection and reliability assessment. Other work uses entropy directly in training or fine-tuning: Li et al.~\citep{li2025preserving} connect diversity preservation in supervised fine-tuning to entropy-regularized objectives, while Shen~\citep{shen2026entropy} studies entropy control in LLM reinforcement learning. These works use entropy mainly for uncertainty estimation, hallucination detection, or training-time regularization. Our focus is different: we compare the conditional entropy of model-generated outputs against the corresponding training distribution and study the resulting conditional output-range gap.

\textbf{Matrix-entropy methods for LLMs.}
Closely related to our use of matrix-based entropy are recent works applying von Neumann or spectral entropy to LLM outputs and representations. Nikitin et al.~\citep{nikitin2024kernel} propose Kernel Language Entropy (KLE), a method for uncertainty quantification that builds positive semidefinite unit-trace kernels over LLM outputs and measures semantic uncertainty using von Neumann entropy. Vendi score~\citep{friedman2023vendi,ospanov2024towards} uses von Neumann entropy of the kernel matrices for evaluating the variety of the generated data, which has been similarly formulated for order-2 Rényi entropy as the RKE score \citep{jalali2023information}. Wei et al.~\citep{wei2024differank} use matrix-entropy-inspired representation spectra for LLM evaluation, showing that rank/entropy-based representation measures correlate with model scale and standard performance indicators. Arefin et al.~\citep{arefin2025seqvcr} use matrix-based entropy to diagnose representation collapse in intermediate Transformer layers and propose a regularizer to increase representation richness for reasoning. These works demonstrate that spectral entropy is informative for LLM uncertainty, representation quality, and training behavior. However, they primarily analyze unconditional output sets or internal representations. On the other hand, our work studies \emph{conditional} entropy and compares model outputs directly to their training data.

\textbf{Entropy regularization and predictive uncertainty.}
Finally, our work is related in spirit to entropy-promoting regularization in supervised learning. Confidence-penalty regularization~\citep{pereyra2017regularizing} and label smoothing~\citep{szegedy2016rethinking} both discourage overly sharp predictive distributions and have been shown to improve deep models across tasks. More recent work has emphasized that real-world supervised datasets may contain substantial ambiguity even when only a single hard label is observed, as reflected for example in human label distributions~\citep{peterson2019human} and critiques of the single-label assumption in ImageNet-style benchmarks~\citep{anzaku2024reassessing}. Although these works are not about generative diversity evaluation, they share a common theme with our analysis: finite datasets can under-represent the uncertainty or variability of the underlying task, and entropy-promoting mechanisms can partially compensate for this mismatch.

\section{Extended Preliminaries}
\label{app:preliminaries_extended}

This appendix records the population-level operator definitions corresponding to the empirical quantities in Section~\ref{sec:preliminaries}. Throughout, $k_X$ and $k_Y$ are normalized positive semidefinite kernels with feature maps
\[
\phi_X:\mathcal X\to\mathcal H_X,
\qquad
\phi_Y:\mathcal Y\to\mathcal H_Y,
\]
so that $k_X(x,x')=\langle \phi_X(x),\phi_X(x')\rangle_{\mathcal H_X}$, $k_Y(y,y')=\langle \phi_Y(y),\phi_Y(y')\rangle_{\mathcal H_Y}$, and
\[
\|\phi_X(x)\|_{\mathcal H_X}^2=1,\qquad
\|\phi_Y(y)\|_{\mathcal H_Y}^2=1.
\]
The product kernel on paired samples is induced by the tensor-product feature map
\[
\psi(x,y)=\phi_X(x)\otimes\phi_Y(y),
\]
since
\[
\langle \psi(x,y),\psi(x',y')\rangle_{\mathcal H_X\otimes\mathcal H_Y}
=
k_X(x,x')k_Y(y,y').
\]

For a joint distribution $P_{XY}$, define the joint and marginal covariance operators
\[
C_{XY}(P)
=
\mathbb E_{(X,Y)\sim P}
\bigl[
\psi(X,Y)\psi(X,Y)^\ast
\bigr],
\qquad
C_X(P_X)
=
\mathbb E_{X\sim P_X}
\bigl[
\phi_X(X)\phi_X(X)^\ast
\bigr].
\]
The normalization of the kernels implies
\[
\operatorname{Tr}(C_{XY}(P))=1,
\qquad
\operatorname{Tr}(C_X(P_X))=1,
\]
so these are trace-one positive semidefinite operators whenever the expectations are well defined. The population conditional von Neumann entropy is
\[
\mathcal H_{\mathrm{vN}}(Y\mid X;P)
:=
H_{\mathrm{vN}}(C_{XY}(P))
-
H_{\mathrm{vN}}(C_X(P_X)).
\]
This is the population operator analogue of the empirical quantity
\[
H_{\mathrm{vN}}\bigl(\frac1nK_X\odot K_Y\bigr)
-
H_{\mathrm{vN}}\bigl(\frac1nK_X\bigr).
\]
Indeed, for the empirical distribution $\widehat P_n=\frac1n\sum_{i=1}^n\delta_{(x_i,y_i)}$, the nonzero eigenvalues of the empirical covariance operator
\[
C_{XY}(\widehat P_n)
=
\frac1n\sum_{i=1}^n
\psi(x_i,y_i)\psi(x_i,y_i)^\ast
\]
coincide with the nonzero eigenvalues of $\frac1nK_{XY}$, where $K_{XY}=K_X\odot K_Y$. Similarly, the nonzero eigenvalues of $C_X(\widehat P_{n,X})$ coincide with those of $\frac1nK_X$. Thus the Gram-matrix formula in the main text is exactly the finite-sample representation of the operator definition.

The order-$2$ version is defined at the population level by
\[
\mathcal H_2(Y\mid X;P)
:=
H_2(C_{XY}(P))-H_2(C_X(P_X)),
\qquad
H_2(A):=-\log\operatorname{Tr}(A^2).
\]
For empirical samples, this becomes
\[
H_2\bigl(\frac1nK_{XY}\bigr)
-
H_2\bigl(\frac1nK_X\bigr)
=
\log\frac{\|K_X\|_F^2}{\|K_X\odot K_Y\|_F^2}.
\]
The equality follows from
\[
\operatorname{Tr}\bigl(\bigl(\frac1nK_X\bigr)^2\bigr)
=
\frac{1}{n^2}\|K_X\|_F^2,
\qquad
\operatorname{Tr}\bigl(\bigl(\frac1nK_{XY}\bigr)^2\bigr)
=
\frac{1}{n^2}\|K_X\odot K_Y\|_F^2.
\]

The normalized-kernel assumption is used throughout to ensure that the covariance operators have unit trace and hence can be interpreted as density operators for von Neumann entropy. If a kernel $k$ satisfies $k(z,z)>0$, the normalized kernel
\[
\widetilde k(z,z')
=
\frac{k(z,z')}{\sqrt{k(z,z)k(z',z')}}
\]
is positive semidefinite and satisfies $\widetilde k(z,z)=1$. In infinite-dimensional settings, the definitions above are understood for trace-class covariance operators with finite entropy; all empirical Gram-matrix quantities used in the experiments are finite-dimensional and therefore well defined.

\section{Proofs and Additional Theoretical Results}
\label{app:proofs_additional_results}


\subsection{Proof of Proposition~\ref{prop:shannon_cond_concave_main}}

\begin{proof}
Let $P^{(0)}_{XY}$ and $P^{(1)}_{XY}$ be two joint distributions on finite alphabets $\mathcal X$ and $\mathcal Y$, and let
\[
P^\lambda_{XY}
=
(1-\lambda)P^{(0)}_{XY}+\lambda P^{(1)}_{XY},
\qquad
\lambda\in[0,1].
\]
We use the convention $0\log 0=0$. For any joint law $P_{XY}$,
\[
H(Y\,\vert\, X;P)
=
-\sum_{x,y}P_{XY}(x,y)
\log\frac{P_{XY}(x,y)}{P_X(x)}.
\]
Fix $x\in\mathcal X$. Applying the log-sum inequality to the two collections
\[
a_y=(1-\lambda)P^{(0)}_{XY}(x,y),
\qquad
b_y=\lambda P^{(1)}_{XY}(x,y)
\]
gives
\[
\sum_y P^\lambda_{XY}(x,y)
\log\frac{P^\lambda_{XY}(x,y)}{P^\lambda_X(x)}
\le
(1-\lambda)
\sum_y P^{(0)}_{XY}(x,y)
\log\frac{P^{(0)}_{XY}(x,y)}{P^{(0)}_X(x)}
+
\lambda
\sum_y P^{(1)}_{XY}(x,y)
\log\frac{P^{(1)}_{XY}(x,y)}{P^{(1)}_X(x)}.
\]
Multiplying by $-1$ and summing over $x$ yields
\[
H(Y\,\vert\, X;P^\lambda)
\ge
(1-\lambda)H(Y\,\vert\, X;P^{(0)})
+
\lambda H(Y\,\vert\, X;P^{(1)}),
\]
which proves concavity.
\end{proof}

\subsection{Proof of Proposition~\ref{thm:conditional_vne_concave_main}}

We first make explicit the operator representation used in the proof. Let $k_X$ and $k_Y$ be normalized positive semidefinite kernels with feature maps
\[
\phi_X:\mathcal X\to\mathcal H_X,
\qquad
\phi_Y:\mathcal Y\to\mathcal H_Y,
\]
satisfying
\[
\|\phi_X(x)\|_{\mathcal H_X}^2=k_X(x,x)=1,
\qquad
\|\phi_Y(y)\|_{\mathcal H_Y}^2=k_Y(y,y)=1.
\]
For a joint distribution $P_{XY}$, define
\[
C_{XY}(P)
=
\mathbb E_{(X,Y)\sim P}
\Bigl[
(\phi_X(X)\otimes\phi_Y(Y))
(\phi_X(X)\otimes\phi_Y(Y))^\ast
\Bigr],
\]
and
\[
C_X(P_X)
=
\mathbb E_{X\sim P_X}
\Bigl[
\phi_X(X)\phi_X(X)^\ast
\Bigr].
\]
The normalized-kernel assumption gives
\[
\operatorname{Tr}(C_{XY}(P))=1,
\qquad
\operatorname{Tr}(C_X(P_X))=1.
\]
Thus, these covariance operators are density operators whenever the relevant entropies are finite. The conditional von Neumann entropy used in the paper is
\[
\mathcal H_{\mathrm{vN}}(Y\,\vert\, X;P)
=
H_{\mathrm{vN}}(C_{XY}(P))
-
H_{\mathrm{vN}}(C_X(P_X)).
\]
For empirical Gram matrices and finite-dimensional sketched features, all operators below are finite-dimensional. For general kernels, the same argument applies under the standard trace-class and finite-entropy conditions for the corresponding density operators.

\begin{proof}
Let
\[
\rho_{XY}(P):=C_{XY}(P).
\]
We first identify the correct marginal operator. For a rank-one tensor term,
\[
(\phi_X(x)\otimes\phi_Y(y))
(\phi_X(x)\otimes\phi_Y(y))^\ast,
\]
taking the partial trace over $\mathcal H_Y$ gives
\[
\operatorname{Tr}_{Y}
\Bigl[
(\phi_X(x)\otimes\phi_Y(y))
(\phi_X(x)\otimes\phi_Y(y))^\ast
\Bigr]
=
\|\phi_Y(y)\|_{\mathcal H_Y}^2
\phi_X(x)\phi_X(x)^\ast.
\]
Since $k_Y(y,y)=1$, this equals $\phi_X(x)\phi_X(x)^\ast$. Taking expectation over $(X,Y)\sim P$ yields
\[
\operatorname{Tr}_Y \rho_{XY}(P)=C_X(P_X).
\]
Therefore, we have
\[
\mathcal H_{\mathrm{vN}}(Y\,\vert\, X;P)
=
H_{\mathrm{vN}}(\rho_{XY}(P))
-
H_{\mathrm{vN}}(\operatorname{Tr}_Y\rho_{XY}(P)).
\]

Now take two joint distributions $P^{(0)}_{XY}$ and $P^{(1)}_{XY}$, and define
\[
P^\lambda_{XY}
=
(1-\lambda)P^{(0)}_{XY}+\lambda P^{(1)}_{XY},
\qquad
\lambda\in[0,1].
\]
The covariance map is affine, so
\[
\rho_{XY}(P^\lambda)
=
(1-\lambda)\rho_{XY}(P^{(0)})
+
\lambda\rho_{XY}(P^{(1)}).
\]
We set the following:
\[
\rho^{(0)}_{XY}:=\rho_{XY}(P^{(0)}),
\qquad
\rho^{(1)}_{XY}:=\rho_{XY}(P^{(1)}),
\qquad
\rho^\lambda_{XY}:=(1-\lambda)\rho^{(0)}_{XY}+\lambda\rho^{(1)}_{XY}.
\]

We consider an auxiliary two-dimensional classical register $R$ with orthonormal basis vectors $e_0,e_1$, and define the block-diagonal density operator
\[
\Omega_{RXY}
=
(1-\lambda)(e_0e_0^\ast)\otimes \rho^{(0)}_{XY}
+
\lambda(e_1e_1^\ast)\otimes \rho^{(1)}_{XY}.
\]
Its $XY$ marginal is $
\Omega_{XY}=\rho^\lambda_{XY}$. Hence,
\[
H_{\mathrm{vN}}(Y\,\vert\, X)_{\Omega}
=
H_{\mathrm{vN}}(\rho^\lambda_{XY})
-
H_{\mathrm{vN}}(\operatorname{Tr}_Y\rho^\lambda_{XY})
=
\mathcal H_{\mathrm{vN}}(Y\,\vert\, X;P^\lambda).
\]

We next compute the conditional entropy after additionally conditioning on the classical register $R$. Since $\Omega_{RXY}$ is block diagonal,
\[
H_{\mathrm{vN}}(\Omega_{RXY})
=
h(\lambda)
+
(1-\lambda)H_{\mathrm{vN}}(\rho^{(0)}_{XY})
+
\lambda H_{\mathrm{vN}}(\rho^{(1)}_{XY}),
\]
where
\[
h(\lambda)=-(1-\lambda)\log(1-\lambda)-\lambda\log\lambda.
\]
Similarly,
\[
\Omega_{RX}
=
(1-\lambda)(e_0e_0^\ast)\otimes \operatorname{Tr}_Y\rho^{(0)}_{XY}
+
\lambda(e_1e_1^\ast)\otimes \operatorname{Tr}_Y\rho^{(1)}_{XY},
\]
and therefore
\[
H_{\mathrm{vN}}(\Omega_{RX})
=
h(\lambda)
+
(1-\lambda)H_{\mathrm{vN}}(\operatorname{Tr}_Y\rho^{(0)}_{XY})
+
\lambda H_{\mathrm{vN}}(\operatorname{Tr}_Y\rho^{(1)}_{XY}).
\]
Subtracting the last two equations results in
\[
H_{\mathrm{vN}}(Y\,\vert\, XR)_{\Omega}
=
(1-\lambda)
\Bigl[
H_{\mathrm{vN}}(\rho^{(0)}_{XY})
-
H_{\mathrm{vN}}(\operatorname{Tr}_Y\rho^{(0)}_{XY})
\Bigr]
+
\lambda
\Bigl[
H_{\mathrm{vN}}(\rho^{(1)}_{XY})
-
H_{\mathrm{vN}}(\operatorname{Tr}_Y\rho^{(1)}_{XY})
\Bigr].
\]
Equivalently,
\[
H_{\mathrm{vN}}(Y\,\vert\, XR)_{\Omega}
=
(1-\lambda)\mathcal H_{\mathrm{vN}}(Y\,\vert\, X;P^{(0)})
+
\lambda\mathcal H_{\mathrm{vN}}(Y\,\vert\, X;P^{(1)}).
\]

Using the strong subadditivity of von Neumann entropy \cite{LiebRuskai1973} implies the monotonicity of quantum conditional entropy under additional conditioning:
\[
H_{\mathrm{vN}}(Y\,\vert\, X)_{\Omega}
\ge
H_{\mathrm{vN}}(Y\,\vert\, XR)_{\Omega}.
\]
Combining the identities above gives
\[
\mathcal H_{\mathrm{vN}}(Y\,\vert\, X;P^\lambda)
\ge
(1-\lambda)\mathcal H_{\mathrm{vN}}(Y\,\vert\, X;P^{(0)})
+
\lambda\mathcal H_{\mathrm{vN}}(Y\,\vert\, X;P^{(1)}),
\]
which finishes the concavity proof for the kernel-induced matrix-based conditional entropy.
\end{proof}

\subsection{Additional Theoretical Results on Finite-Sample Conditional Entropy}
\label{app:finite_sample_conditional_entropy}

\begin{corollary}[Finite-sample underestimation and monotonicity for Shannon conditional entropy]
\label{cor:shannon_cond_bias_monotone_app}
Let $\widehat P_n$ be the empirical joint distribution of $n$ i.i.d.\ samples from a distribution $P_{XY}$ on finite alphabets. Then
\[
\mathbb E[H(Y\,\vert\, X;\widehat P_n)]
\le
H(Y\,\vert\, X;P).
\]
Moreover, for all integers $m\le n$,
\[
\mathbb E[H(Y\,\vert\, X;\widehat P_m)]
\le
\mathbb E[H(Y\,\vert\, X;\widehat P_n)].
\]
\end{corollary}

\begin{proof}
The first inequality follows from Jensen's inequality and Proposition~\ref{prop:shannon_cond_concave_main}:
\[
\mathbb E[H(Y\,\vert\, X;\widehat P_n)]
\le
H(Y\,\vert\, X;\mathbb E[\widehat P_n])
=
H(Y\,\vert\, X;P).
\]

For monotonicity, let $Z_1,\ldots,Z_n$ be the paired samples, where $Z_i=(X_i,Y_i)$. Conditional on $Z_1,\ldots,Z_n$, draw a subset $S\subseteq\{1,\ldots,n\}$ of size $m$ uniformly without replacement, and let $\widehat P_S$ be the empirical distribution of the selected samples. Then
\[
\mathbb E[\widehat P_S\,\vert\, Z_1,\ldots,Z_n]=\widehat P_n.
\]
The concavity results in the following:
\[
H(Y\,\vert\, X;\widehat P_n)
=
H\Bigl(Y\,\vert\, X;\mathbb E[\widehat P_S\,\vert\, Z_1,\ldots,Z_n]\Bigr)
\ge
\mathbb E\Bigl[
H(Y\,\vert\, X;\widehat P_S)
\,\vert\, Z_1,\ldots,Z_n
\Bigr].
\]
Taking expectations gives
\[
\mathbb E[H(Y\,\vert\, X;\widehat P_n)]
\ge
\mathbb E[H(Y\,\vert\, X;\widehat P_S)].
\]
The unconditional law of $\widehat P_S$ is the same as the empirical distribution of $m$ i.i.d.\ samples from $P_{XY}$. Therefore
\[
\mathbb E[H(Y\,\vert\, X;\widehat P_n)]
\ge
\mathbb E[H(Y\,\vert\, X;\widehat P_m)].
\]
\end{proof}

\begin{theorem}[First-order bias of empirical Shannon conditional entropy]
\label{thm:shannon_cond_miller_app}
Assume $X$ and $Y$ take values in finite alphabets, and define
\[
\mathcal S_{XY}=\{(x,y):P_{XY}(x,y)>0\},
\qquad
\mathcal S_X=\{x:P_X(x)>0\}.
\]
Using natural logarithms,
\[
H(Y\,\vert\, X;P)
-
\mathbb E[H(Y\,\vert\, X;\widehat P_n)]
=
\frac{|\mathcal S_{XY}|-|\mathcal S_X|}{2n}
+
O\Bigl(\frac{1}{n^2}\Bigr).
\]
\end{theorem}

\begin{proof}
For a discrete distribution $P$ with support size $s$ and strictly positive probabilities on its support, the Miller--Basharin expansion gives
\[
H(P)-\mathbb E[H(\widehat P_n)]
=
\frac{s-1}{2n}
+
O\Bigl(\frac{1}{n^2}\Bigr).
\]
Applying this expansion to $P_{XY}$ gives
\[
H(X,Y;P)
-
\mathbb E[H(X,Y;\widehat P_n)]
=
\frac{|\mathcal S_{XY}|-1}{2n}
+
O\Bigl(\frac{1}{n^2}\Bigr).
\]
Applying the same expansion to the marginal $P_X$ gives
\[
H(X;P_X)
-
\mathbb E[H(X;\widehat P_{n,X})]
=
\frac{|\mathcal S_X|-1}{2n}
+
O\Bigl(\frac{1}{n^2}\Bigr).
\]
Since
\[
H(Y\,\vert\, X;P)=H(X,Y;P)-H(X;P_X)
\]
and
\[
H(Y\,\vert\, X;\widehat P_n)
=
H(X,Y;\widehat P_n)-H(X;\widehat P_{n,X}),
\]
subtracting the marginal expansion from the joint expansion yields
\[
H(Y\,\vert\, X;P)
-
\mathbb E[H(Y\,\vert\, X;\widehat P_n)]
=
\frac{|\mathcal S_{XY}|-|\mathcal S_X|}{2n}
+
O\Bigl(\frac{1}{n^2}\Bigr).
\]
\end{proof}

\begin{corollary}[Finite-sample underestimation and monotonicity for conditional von Neumann entropy]
\label{cor:conditional_vne_bias_monotone_app}
Let $\widehat P_n$ be the empirical distribution of $n$ i.i.d.\ paired samples from $P_{XY}$. Assume the relevant conditional von Neumann entropies are finite. Then
\[
\mathbb E[
\mathcal H_{\mathrm{vN}}(Y\,\vert\, X;\widehat P_n)
]
\le
\mathcal H_{\mathrm{vN}}(Y\,\vert\, X;P).
\]
Moreover, for all integers $m\le n$,
\[
\mathbb E[
\mathcal H_{\mathrm{vN}}(Y\,\vert\, X;\widehat P_m)
]
\le
\mathbb E[
\mathcal H_{\mathrm{vN}}(Y\,\vert\, X;\widehat P_n)
].
\]
\end{corollary}

\begin{proof}
The first inequality follows from Jensen's inequality and Proposition~\ref{thm:conditional_vne_concave_main}:
\[
\mathbb E[
\mathcal H_{\mathrm{vN}}(Y\,\vert\, X;\widehat P_n)
]
\le
\mathcal H_{\mathrm{vN}}(Y\,\vert\, X;\mathbb E[\widehat P_n])
=
\mathcal H_{\mathrm{vN}}(Y\,\vert\, X;P).
\]

For monotonicity, condition on $n$ paired samples $Z_1,\ldots,Z_n$ and draw a subset $S$ of size $m$ uniformly without replacement. As above,
\[
\mathbb E[\widehat P_S\,\vert\, Z_1,\ldots,Z_n]=\widehat P_n.
\]
By concavity,
\[
\mathcal H_{\mathrm{vN}}(Y\,\vert\, X;\widehat P_n)
\ge
\mathbb E[
\mathcal H_{\mathrm{vN}}(Y\,\vert\, X;\widehat P_S)
\,\vert\, Z_1,\ldots,Z_n].
\]
Taking expectations and using the fact that $\widehat P_S$ has the same law as an empirical distribution formed from $m$ i.i.d.\ samples proves the result.
\end{proof}

\subsection{Additional Theoretical Results on Lifted Bregman Divergences}

Fix an input marginal $P_X$. The map
\[
R_{Y\,\vert\, X}\longmapsto P_XR_{Y\,\vert\, X}
\]
embeds a conditional law into a joint law whose input marginal is fixed. Given a Bregman divergence $D_\Phi$ on a convex domain of joint laws, define
\[
D_{\Phi\,\vert\, P_X}(R_{Y\,\vert\, X},Q_{Y\,\vert\, X})
=
D_\Phi(P_XR_{Y\,\vert\, X},P_XQ_{Y\,\vert\, X}).
\]

If $D_\Phi$ is the KL divergence, then
\[
D_{\Phi\,\vert\, P_X}(R_{Y\,\vert\, X},Q_{Y\,\vert\, X})
=
\mathbb E_{X\sim P_X}
\Bigl[
\mathrm{KL}\bigl(R_{Y\,\vert\, X}(\cdot\,\vert\, X)\,\Vert\,Q_{Y\,\vert\, X}(\cdot\,\vert\, X)\bigr)
\Bigr],
\]
whenever $P_XR_{Y\,\vert\, X}$ is absolutely continuous with respect to $P_XQ_{Y\,\vert\, X}$.

If $\psi(x,y)$ is a feature map on joint pairs and
\[
\mu_{P_XR}
=
\mathbb E_{X\sim P_X,\,Y\sim R(\cdot\,\vert\, X)}[\psi(X,Y)]
\]
is the kernel mean embedding of $P_XR_{Y\,\vert\, X}$, then
\[
D_{\mathrm{MMD}\,\vert\, P_X}(R_{Y\,\vert\, X},Q_{Y\,\vert\, X})
=
\|\mu_{P_XR}-\mu_{P_XQ}\|_{\mathcal H}^2.
\]
This is the Bregman divergence generated by $\Phi(\mu)=\|\mu\|_{\mathcal H}^2$ on the Hilbert space of mean embeddings, since
\[
D_\Phi(\mu,\nu)=\|\mu-\nu\|_{\mathcal H}^2.
\]

\subsection{Proof of Theorem~\ref{thm:conditional_entropy_projection_principle}}

For $\rho\in\mathbb R$, we consider the definition:
\[
\mathcal C_\rho(P_X)
=
\Bigl\{
R_{Y\,\vert\, X}:
\mathcal H_{\mathrm{vN}}(R_{Y\,\vert\, X};P_X)\ge \rho
\Bigr\}.
\]

\begin{lemma}[Convexity of conditional entropy superlevel sets]
\label{lem:conditional_superlevel_convex}
The set $\mathcal C_\rho(P_X)$ is convex.
\end{lemma}

\begin{proof}
Let $R^{(0)}_{Y\,\vert\, X},R^{(1)}_{Y\,\vert\, X}\in\mathcal C_\rho(P_X)$ and $\lambda\in[0,1]$. Define
\[
R^\lambda_{Y\,\vert\, X}
=
(1-\lambda)R^{(0)}_{Y\,\vert\, X}
+
\lambda R^{(1)}_{Y\,\vert\, X}.
\]
Because the input marginal is fixed,
\[
P_XR^\lambda_{Y\,\vert\, X}
=
(1-\lambda)P_XR^{(0)}_{Y\,\vert\, X}
+
\lambda P_XR^{(1)}_{Y\,\vert\, X}.
\]
By Proposition~\ref{thm:conditional_vne_concave_main},
\[
\mathcal H_{\mathrm{vN}}(R^\lambda_{Y\,\vert\, X};P_X)
\ge
(1-\lambda)\mathcal H_{\mathrm{vN}}(R^{(0)}_{Y\,\vert\, X};P_X)
+
\lambda\mathcal H_{\mathrm{vN}}(R^{(1)}_{Y\,\vert\, X};P_X)
\ge \rho.
\]
Thus $R^\lambda_{Y\,\vert\, X}\in\mathcal C_\rho(P_X)$.
\end{proof}

We also use the following lemma on Bregman divergences. Let $\Omega$ be a convex domain, let $\Phi:\Omega\to\mathbb R$ be differentiable and strictly convex, and define
\[
D_\Phi(A,B)
=
\Phi(A)-\Phi(B)-\langle \nabla\Phi(B),A-B\rangle.
\]

\begin{lemma}
\label{lem:bregman_pythagorean_conditional}
Let $\mathcal C\subseteq\Omega$ be nonempty and convex. Suppose
\[
Q^\star\in \underset{R\in\mathcal C}{\arg\!\min}
\: D_\Phi(R,Q).
\]
Then, for every $P\in\mathcal C$,
\[
D_\Phi(P,Q^\star)+D_\Phi(Q^\star,Q)\le D_\Phi(P,Q).
\]
\end{lemma}

\begin{proof}
The first-order optimality condition for the projection gives
\[
\langle \nabla\Phi(Q^\star)-\nabla\Phi(Q),P-Q^\star\rangle\ge 0,
\qquad P\in\mathcal C.
\]
The three-point identity for Bregman divergences gives
\[
D_\Phi(P,Q)
=
D_\Phi(P,Q^\star)
+
D_\Phi(Q^\star,Q)
+
\langle \nabla\Phi(Q^\star)-\nabla\Phi(Q),P-Q^\star\rangle.
\]
Combining the two displays yields the inequality.
\end{proof}

\begin{proof}[Proof of Theorem~\ref{thm:conditional_entropy_projection_principle}]
The choice
\[
\rho=\mathcal H_{\mathrm{vN}}(P_{Y\,\vert\, X};P_X)
\]
implies $P_{Y\,\vert\, X}\in\mathcal C_\rho(P_X)$. By Lemma~\ref{lem:conditional_superlevel_convex}, $\mathcal C_\rho(P_X)$ is convex. Applying Lemma~\ref{lem:bregman_pythagorean_conditional} to the lifted joint laws
$
P_XP_{Y\,\vert\, X},\,
P_XQ_{Y\,\vert\, X},
\,
P_XQ^\star_{Y\,\vert\, X}$
over the convex set
$
\{P_XR_{Y\,\vert\, X}:R_{Y\,\vert\, X}\in\mathcal C_\rho(P_X)\}$ results in
\[
D_{\Phi\,\vert\, P_X}(P_{Y\,\vert\, X},Q^\star_{Y\,\vert\, X})
+
D_{\Phi\,\vert\, P_X}(Q^\star_{Y\,\vert\, X},Q_{Y\,\vert\, X})
\le
D_{\Phi\,\vert\, P_X}(P_{Y\,\vert\, X},Q_{Y\,\vert\, X}).
\]
Therefore, the proof is complete.
\end{proof}

\subsection{Additional Theoretical Results on Finite-Support Conditional MMD Projection}

Let $x_1,\ldots,x_N$ be input prompts, and let $y_{i,1},\ldots,y_{i,m}$ be generated candidates for $x_i$. Let $M=Nm$ and index pairs by $a=(i,j)$:
\[
s_a=(x_i,y_{i,j}).
\]
Let $q\in\mathbb R^M$ be a distribution over the generated pairs and let $u\in\mathbb R^M$ be the uniform baseline,
\[
u_{i,j}=\frac{1}{Nm}.
\]
The conditional block constraints are
\[
q_{i,j}\ge 0,
\qquad
\sum_{j=1}^m q_{i,j}=\frac1N,
\qquad i=1,\ldots,N.
\]
Equivalently, $q_{i,j}=p_{i,j}/N$ with $p_i\in\Delta_m$.

Let
\[
k((x,y),(x',y'))=k_X(x,x')k_Y(y,y')
\]
be the product kernel, and let $K\in\mathbb R^{M\times M}$ be the Gram matrix over the generated pairs. Then the squared MMD between the reweighted distribution $q$ and the uniform baseline $u$ is
\[
\Bigl\|
\sum_a(q_a-u_a)\psi(s_a)
\Bigr\|_{\mathcal H}^2
=
(q-u)^\top K(q-u).
\]
Thus, the objective in~\eqref{eq:kernel_conditional_projection} is the finite-support instance of the lifted conditional MMD with empirical input marginal.

The weighted joint covariance operator is
\[
C_{XY}(q)
=
\sum_a q_a\psi(s_a)\psi(s_a)^\ast.
\]
Because the product kernel is normalized and $\sum_a q_a=1$, this operator has trace one. The empirical input covariance is
\[
C_X
=
\frac1N\sum_{i=1}^N\phi_X(x_i)\phi_X(x_i)^\ast.
\]
The block constraints make $C_X$ independent of $q$, so
\[
\mathcal H_{\mathrm{vN}}(q)
=
H_{\mathrm{vN}}(C_{XY}(q))-H_{\mathrm{vN}}(C_X).
\]

The nonzero eigenvalues of $C_{XY}(q)$ coincide with the nonzero eigenvalues of the weighted Gram matrix
\[
K_q
=
\operatorname{diag}(\sqrt q)\,K\,\operatorname{diag}(\sqrt q).
\]
Indeed, writing $C_{XY}(q)=AA^\ast$ with columns $\sqrt{q_a}\psi(s_a)$ gives $K_q=A^\ast A$, and $AA^\ast$ and $A^\ast A$ have the same nonzero spectrum.

\subsection{Proof of Proposition~\ref{prop:finite_support_projection}}

\begin{proof}
The objective $(q-u)^\top K(q-u)$ is convex because $K$ is positive semidefinite. The nonnegativity and block-marginal constraints are linear. The map $q\mapsto C_{XY}(q)$ is affine, and $H_{\mathrm{vN}}$ is concave on positive trace-one operators. Since $C_X$ is fixed over the feasible set, $q\mapsto\mathcal H_{\mathrm{vN}}(q)$ is concave. Hence, the entropy superlevel constraint is convex, and the program is convex.
\end{proof}

\subsection{Proof of Proposition~\ref{prop:covariance_space_projection}}

\begin{proof}
The squared norm term is convex because it is the squared norm of an affine function of $p$. The map $p\mapsto \widetilde C(p)$ is affine, and von Neumann entropy is concave on trace-one positive semidefinite matrices. Therefore $p\mapsto -H_{\mathrm{vN}}(\widetilde C(p))$ is convex. The constraints $p_i\in\Delta_m$ are linear, so the problem is convex.
\end{proof}

\subsection{Proof of Theorem~\ref{thm:cep_beg_convergence_main}}

CEP-BEG uses the product-simplex negative-entropy mirror map
\[
\Psi(p)=\sum_{i=1}^N\sum_{j=1}^m p_{i,j}\log p_{i,j}.
\]
The associated Bregman divergence is
\[
D_\Psi(p,p')
=
\sum_{i=1}^N\sum_{j=1}^m
p_{i,j}\log\frac{p_{i,j}}{p'_{i,j}}.
\]
The mirror step is
\[
p^{(t+1)}
=
\underset{p_1,\ldots,p_N\in\Delta_m}{\arg\!\min}\:
\Bigl\{
\langle g^{(t)},p\rangle
+
\frac1\eta D_\Psi(p,p^{(t)})
\Bigr\},
\]
which yields the block exponentiated-gradient update in Algorithm~\ref{alg:cep_beg}.

\begin{proof}
For each block $i$, the negative-entropy mirror descent inequality gives
\[
\langle g_i^{(t)},p_i^{(t)}-p_i\rangle
\le
\frac{
D_{\mathrm{KL}}(p_i,p_i^{(t)})
-
D_{\mathrm{KL}}(p_i,p_i^{(t+1)})
}{\eta}
+
\frac{\eta}{2}\|g_i^{(t)}\|_\infty^2.
\]
Summing over $i=1,\ldots,N$ yields
\[
\langle g^{(t)},p^{(t)}-p\rangle
\le
\frac{
D_\Psi(p,p^{(t)})
-
D_\Psi(p,p^{(t+1)})
}{\eta}
+
\frac{\eta}{2}
\sum_{i=1}^N\|g_i^{(t)}\|_\infty^2.
\]
Taking $p=p^\star$ and using convexity of $F$,
\[
F(p^{(t)})-F(p^\star)
\le
\langle g^{(t)},p^{(t)}-p^\star\rangle.
\]
Summing over $t=0,\ldots,T-1$ gives
\[
\sum_{t=0}^{T-1}\big(F(p^{(t)})-F(p^\star)\big)
\le
\frac{D_\Psi(p^\star,p^{(0)})}{\eta}
+
\frac{\eta T G^2}{2}.
\]
Since $p^{(0)}_{i,j}=1/m$ and $p_i^\star\in\Delta_m$,
\[
D_\Psi(p^\star,p^{(0)})
=
\sum_{i=1}^N\sum_{j=1}^m
p^\star_{i,j}\log(mp^\star_{i,j})
\le
N\log m.
\]
Dividing by $T$ and using Jensen's inequality for the convex function $F$,
\[
F(\bar p^{(T)})
\le
\frac1T\sum_{t=0}^{T-1}F(p^{(t)}),
\]
we obtain
\[
F(\bar p^{(T)})-F(p^\star)
\le
\frac{N\log m}{\eta T}
+
\frac{\eta G^2}{2}.
\]
The stated choice of $\eta=\sqrt{2N\log m}/(G\sqrt T)$ minimizes this upper bound which becomes the following and finishes the proof: 
\[
F(\bar p^{(T)})-F(p^\star)
\le
G\sqrt{\frac{2N\log m}{T}}.
\]
The proof is hence complete.
\end{proof}

For the regularized entropy implementation, a concrete gradient bound follows under the assumption of a spectral floor. If $\|\tilde z_{i,j}\|_2=1$ and the eigenvalues of the matrices used inside the logarithm are bounded below by $\tau>0$, then
\[
\|g_i^{(t)}\|_\infty
\le
\frac{4+\lambda(1+|\log\tau|)}{N}
\quad
\text{for all }i,t.
\]
Consequently, one may take
\[
G
=
\frac{4+\lambda(1+|\log\tau|)}{\sqrt N}.
\]
The bound follows from $\|\tilde r(p)\|_2\le 2$ and
\[
\|\log \widetilde C(p)+I\|_{\mathrm{op}}
\le
1+|\log\tau|
\]
on the spectral interval $[\tau,1]$.

\section{Implementation Details}
\label{app:implementation_details}

\subsection{LLM Data Construction and Sample Generation}
\label{app:llm_generation}

We construct paired prefix--continuation samples from the corresponding open training corpora of the evaluated language models. For OLMo, we use the Dolma corpus; for Pythia and GPT-Neo, we use The Pile. Each example is represented as a pair $(X,Y)$, where $X$ is a short prefix used as the conditioning input and $Y$ is the following continuation. The same prefix set is used for both the training-data reference and the model-generated continuations, so the input marginal is matched across groups.

For OLMo, we use the \texttt{allenai/OLMo-1B-hf} checkpoint and its associated tokenizer. The tokenizer padding token is set to the end-of-sequence token when no padding token is provided, and left padding is used for batched autoregressive generation. The model is loaded with automatic device mapping and half precision.
To construct the training reference for OLMo, we use the Hugging Face Dolma dataset. We load the downloaded Dolma files with the Hugging Face \texttt{datasets} library and process each document into paragraph-level candidates. We split documents by blank lines, normalize whitespace, and keep only paragraphs with at least $500$ characters and at least $3$ sentence-ending marks. From each document, we uniformly sample the candidate paragraphs. We then tokenize each candidate paragraph and keep the first paragraph that contains at least the required number of tokens. Unless otherwise specified, we use a prefix length of $3$ tokens and a continuation length of $2$ tokens. The prefix tokens are used as the condition $X$, and the immediately following tokens are used as the training continuation $Y_{\mathrm{train}}$. To examine the effect of token-length choices, we additionally conduct an ablation over all combinations of prefix lengths in $\{3,4,5\}$ and continuation lengths in $\{2,3,4\}$. We also conduct a sample-size ablation by drawing subsamples of different sizes under multiple random seeds and reporting the resulting variation across runs. Documents without a valid paragraph of sufficient length are skipped.

For each prefix, we generate model continuations using three decoding strategies. Greedy decoding selects the most likely token at each step. Nucleus sampling samples from the smallest token set whose cumulative probability mass exceeds $p=0.9$. Ancestral sampling samples from the full next-token distribution, implemented by setting \texttt{top\_p=1.0}. All strategies use the same maximum number of newly generated tokens as the training continuation length. 
We store the generated data in JSON format with fields for the example id, prefix, training continuation, greedy decoding output, nucleus sampling output, and ancestral sampling output. In the main experiments, we sample up to $30$K prefix--continuation pairs and evaluate conditional diversity at different sample sizes using subsets of these generated records.
For Pythia and GPT-Neo, we follow the same construction using The Pile as the training-data reference. We use the Hugging Face checkpoints \texttt{EleutherAI/pythia-1b} and \texttt{EleutherAI/gpt-neo-1.3B}, together with their corresponding tokenizers. For each model, we tokenize documents from The Pile, extract matched prefix--continuation pairs, and generate continuations from the same prefix set under greedy decoding, nucleus sampling with the same $p=0.9$, and ancestral sampling. This ensures that the comparison between training and generated groups is always performed under a shared empirical input marginal.

\textbf{Long-sequence construction.}
For the long-sequence experiments in Section~\ref{sec:long_sequence_gap}, we extend the prefix length to $\{8,16,32\}$ tokens and the continuation length to $\{16,32,64\}$ tokens. To represent a long sequence while preserving local information, we split each sequence into $2$-token segments, compute the Qwen3-Embedding of each segment independently, and concatenate the resulting embeddings into a single feature vector. The same chunking construction has been widely used in the long-text literature~\cite{katz2023identifying,jaiswal2023breaking,gunther2024late,zhang2022multiview,santhanam2022colbertv2}. The same prefix set is used for the training-data reference and the model-generated continuations. We additionally perform a chunk-size ablation over chunk sizes $\{1,2,4\}$ in Appendix~\ref{app:chunk_size_ablation}. The entropy computation depends primarily on the number of samples and the representation dimension rather than directly on the number of tokens; longer sequences mainly increase the cost of constructing their representations, after which the same kernel-based estimator and optimization procedure apply.

In Table~\ref{tab:long_sequence_gap} we observe two consistent trends in the gap. First, for a fixed prefix length, the gap generally grows with the continuation length, because a longer continuation leaves more room for valid variation that the model under-covers. Second, for a fixed continuation length, the gap generally shrinks as the prefix length increases, because stronger conditioning reduces the range of plausible continuations. Nevertheless, no configuration eliminates the gap.

\subsection{Image Data Construction and Sample Generation}
\label{app:image_generation}

We construct paired condition--image samples for two conditional image-generation settings: class-conditioned ImageNet generation and text-conditioned MS-COCO generation. Each example is represented as a pair $(X,Y)$, where $X$ is the conditioning variable and $Y$ is the corresponding image. For ImageNet, $X$ is the class label and $Y$ is the image. For MS-COCO, $X$ is the text caption and $Y$ is the image. In both settings, we compare generated images with real-data references under matched conditions.

For the class-conditioned ImageNet experiments, we follow the DGM-Eval benchmark and use the generated image samples provided for several class-conditioned generators, including ADM, LDM, BigGAN, and DiT-XL-2. The real-data reference is constructed from ImageNet training images with their corresponding class labels. For each generated sample, we use its class label as the condition and the generated image as the output. We then form paired samples $(X_i,Y_i)$ for both the real ImageNet reference and each generator. Conditional diversity is evaluated by comparing the real and generated paired distributions under the same class-conditioned protocol.

For the text-conditioned MS-COCO experiments, we use captions sampled from the MS-COCO dataset as conditioning inputs. Given a caption $X_i$, we generate an image $Y_i$ using each text-to-image model. We evaluate U-ViT-S/2 and its corresponding Deep variant, as well as SDXL-Base-1.0 and PixArt-$\Sigma$. For the real-data reference, we use the MS-COCO image associated with the sampled caption. This gives paired caption--image samples for the real MS-COCO reference and for each generated model group. The same caption set is used when comparing the reference and generated samples, so the empirical input marginal over captions is matched.

For all image-generation experiments, conditional VNE and conditional RKE are computed from paired condition--image samples using the product kernel between condition features and image features. In the ImageNet setting, the condition feature is derived from the class label; in the MS-COCO setting, the condition feature is derived from the caption. Unless otherwise specified, we use the gaussian kernel with bandwidth selection following the literature.

For sample-size analysis, we evaluate conditional diversity over subsets of different sizes. In the ImageNet and MS-COCO gap experiments, we report the exponential of conditional VNE over sample sizes ranging from $2.5$K to $20$K. For each sample size, we draw subsamples under multiple random seeds and report the mean and standard deviation across evaluations. This protocol is applied consistently to the real-data reference and to each generated model group.

\subsection{Evaluation Settings and Hyperparameters}
\label{app:evaluation_hyperparams}
For the LLM diversity-gap experiments, the default text representation is Qwen3-Embedding. We compute embeddings for the prefix and continuation texts separately, construct a product kernel over the paired prefix--continuation samples, and evaluate conditional VNE and conditional RKE on the resulting normalized kernel matrices. Unless otherwise stated, we use the Gaussian kernel with bandwidth selected by the median heuristic on the corresponding feature distances. We also report ablations with T5 embeddings, cosine kernels, and degree-3 polynomial kernels with $\gamma=1$ in Tables~\ref{tab:diversity_metrics_gap_20k_t5_long}, \ref{tab:diversity_metrics_gap_20k_poly3_long}, and \ref{tab:diversity_metrics_gap_20k_cosine_long}.

For the image experiments, image features are extracted using the default visual encoder DINOv2 and condition text features are extracted by CLIP model. We use the same normalized product-kernel construction for all conditional VNE and conditional RKE evaluations, so differences in the reported scores reflect the conditional output distributions under matched input conditions.

For entropy-projected reweighting, each prefix is associated with $10$ generated candidate continuations. The optimization is performed over a product of simplices, with one simplex for the candidates associated with each prefix, so the empirical prefix marginal is fixed throughout optimization. We initialize the weights uniformly, use the conditional entropy objective described in Section~\ref{sec:conditional_reweighting}, and report the reweighted empirical distribution using the same evaluation pipeline as the unweighted generated samples. For scalability, the joint product features are computed through the sketched representation described in the method section.

For conditional VNE guidance in diffusion sampling, we use SDXL on MS-COCO captions and set the guidance scale to $\eta=0.03$. The guidance step is applied during denoising after the standard SDXL update. All guided and unguided comparisons use the same caption set and the same evaluation features, and conditional VNE is reported as mean $\pm$ standard deviation over independent subsamples. All experiments were run on 2 RTX-4090 GPUs.

For the MBR decoding experiments in Section~\ref{sec:cvs_mbr}, we follow the candidate-generation protocol of Jinnai et al.~\cite{jinnai2024mbmbmbr}, using nucleus sampling with $p=0.9$ for a controlled comparison. We evaluate XSum with BART-large-XSum, CNN/DailyMail with BART-large-CNN, and CNN/DailyMail with Mistral-7B-Instruct-v0.1 using the CNN/DM prompt from their Appendix G. For each of $50$ test inputs, we compare Monte Carlo MBR, model-based MBR with and without length normalization, and the proposed conditional-Vendi-weighted variant (CVS-MBR). All methods operate on the same candidate pool and utility function and differ only in the reference weighting used in the MBR expectation.

The methods compared in Table~\ref{tab:cvs_mbr} represent different operating points. MBMBR and MBMBR-L use model-probability weights and are mainly designed to improve expected generation quality; accordingly, they achieve higher ROUGE-L and BERTScore, with length normalization mitigating sequence-length bias. In contrast, CVS-MBR explicitly reduces concentration in the reference distribution and therefore achieves the highest conditional VNE and Distinct-$2$, as well as the lowest Self-BLEU-$2$, across all three settings. To examine the quality--diversity trade-off, we performed a paired bootstrap over the $50$ test inputs with $2000$ resamples, comparing CVS-MBR against MBMBR-L. The $95\%$ confidence intervals for the ROUGE-L and BERTScore differences contain zero in all three settings, indicating no statistically significant degradation in quality while CVS-MBR consistently attains the highest conditional VNE.

\section{Additional Numerical Results}

\subsection{Additional LLM Diversity Gap Results}

Tables~\ref{tab:diversity_metrics_gap_10k_long} and \ref{tab:diversity_metrics_gap_20k_full_appendix} provide the full LLM diversity-gap results at sample sizes $10$K and $20$K. Across OLMo, Pythia, and GPT-Neo, the training continuations consistently achieve higher conditional VNE and conditional RKE than model-generated continuations from the same prefixes. The ordering among decoding strategies is also stable: greedy decoding has the largest gap, nucleus sampling reduces the gap, and ancestral sampling is usually the closest to the training reference. For example, at $20$K samples, the conditional VNE gap for OLMo is $119.70$ under greedy decoding, $51.27$ under nucleus sampling, and $41.27$ under ancestral sampling.

The lexical diversity metrics show the same qualitative pattern. Distinct-2 and Distinct-3 are substantially lower for greedy decoding than for the training continuations, while stochastic decoding recovers much of the surface-level n-gram diversity. However, even when Distinct scores become relatively close to the training reference, the conditional VNE and conditional RKE gaps remain positive. This indicates that the proposed kernel-based conditional diversity metrics capture conditional output-range differences that are not fully explained by local n-gram variety.

Comparing $10$K and $20$K samples further shows that the measured conditional diversity gap becomes more pronounced as the evaluation set grows. The training reference gains conditional VNE faster than the generated groups, which is consistent with the sample-size curves in Figure~\ref{fig:llm_sample_size_gap}. This behavior suggests that the training data continues to reveal additional valid conditional continuations as more examples are included, whereas the generated samples cover a narrower conditional range.

\subsection{Additional Image Diversity Gap Results}

Figure~\ref{fig:imagenet_gap} and Table~\ref{tab:imagenet_cvs_full_transposed} report the ImageNet class-conditioned results. The ImageNet reference has the highest conditional VNE at every sample size. Among the evaluated generators, DiT-XL-2 and LDM are closest to the reference, while ADM and BigGAN exhibit much larger gaps. At $20$K samples, ImageNet reaches $54.15$, compared with $49.95$ for DiT-XL-2, $47.86$ for LDM, $22.47$ for ADM, and $17.04$ for BigGAN. Thus, the strongest diffusion and transformer-based generators approach the reference more closely, but they still do not fully match the conditional diversity of real images under the same class labels.

Figure~\ref{fig:coco_gap} and Table~\ref{tab:mscoco_cvs_full_transposed} show the text-conditioned MS-COCO results. The MS-COCO reference again achieves the highest conditional VNE across all sample sizes. SDXL is the strongest among the evaluated generated groups in this experiment, followed closely by U-ViT Deep S/2 and U-ViT S/2, while PixArt-$\Sigma$ has the largest gap. At $20$K samples, the reference reaches $40.52$, whereas SDXL reaches $30.25$, U-ViT Deep S/2 reaches $29.94$, U-ViT S/2 reaches $29.50$, and PixArt-$\Sigma$ reaches $25.25$.

The image results mirror the LLM results: the real-data reference grows faster with sample size than the generated groups, so the diversity gap becomes clearer as more samples are evaluated. Because the comparisons are made under matched class labels or captions, the gap reflects differences in conditional output variability rather than a mismatch in the empirical input marginal.

\subsection{Additional Reweighting Results}

The entropy-projected reweighting results in Table~\ref{tab:projection_conditional_diversity} show that conditional diversity can be increased post hoc by redistributing probability mass over already generated candidates. The reweighting improves conditional VNE for all three LLMs at both evaluated sample sizes. At $20$K samples, conditional VNE increases from $366.35$ to $376.77$ for OLMo, from $371.72$ to $382.74$ for Pythia, and from $340.90$ to $353.68$ for GPT-Neo.

The effect is especially strong for conditional RKE. At $20$K samples, RKE increases from $241.14$ to $284.92$ for OLMo, from $236.24$ to $285.05$ for Pythia, and from $246.13$ to $296.02$ for GPT-Neo. These gains are obtained without retraining the model and without adding new generated continuations. Since the projection preserves the total mass assigned to each prefix, the improvement comes from changing the conditional distribution over available continuations rather than altering the input distribution.

These results support the interpretation that generated candidate sets often contain underweighted alternatives. The projection can recover part of this latent conditional diversity by moving mass toward candidates that increase the joint condition--output entropy while remaining within the fixed candidate support.

\subsection{Additional Conditional VNE Guidance Results}
\label{app:conditional_vne_guidance}

We further evaluate whether the conditional VNE score can be used not only as an evaluation metric, but also as a sampling-time guidance signal for conditional generation. The goal is to encourage higher conditional diversity during sampling while keeping the text condition fixed. 

Let $c^{(n)}$ denote the caption for the $n$-th generated image, and let $z_t^{(n)}$ be its latent variable at diffusion step $t$. Given previously generated caption-image pairs $\{(c^{(i)}, z^{(i)})\}_{i=1}^{n-1}$ and the current pair $(c^{(n)}, z_t^{(n)})$, we construct a caption kernel $K_C$ and a latent-image kernel $K_Z$. The conditional VNE guidance score is
\[
\widehat{H}_{\mathrm{vN}}(Z \mid C)
=
H_{\mathrm{vN}}\!\left(\frac{1}{n}(K_C \odot K_Z)\right)
-
H_{\mathrm{vN}}\!\left(\frac{1}{n}K_C\right),
\]
where $\odot$ denotes the Hadamard product. Since the caption kernel $K_C$ is fixed during the latent update, the guidance acts through the gradient of the joint conditional term with respect to the current latent variable. After the standard SDXL denoising update, we apply an additional guidance step
\[
z_{t-1}^{(n)}
\leftarrow
z_{t-1}^{(n)}
+
\eta \nabla_{z_{t-1}^{(n)}}
\widehat{H}_{\mathrm{vN}}(Z \mid C),
\]
where $\eta$ is the guidance scale. 
For the experiment, we use captions from the MS-COCO dataset and generate $30$K images with SDXL, both with and without the proposed guidance, and the guidance scale we used is $0.03$. We then evaluate conditional VNE at sample sizes from $10$K to $20$K reported in
Table~\ref{tab:mscoco_sdxl_guidance_cvs}. SDXL with conditional VNE guidance consistently obtains higher conditional VNE than SDXL without guidance at every evaluated sample size. For example, at $20$K samples, the conditional VNE increases from $30.25$ without guidance to $32.27$ with guidance. The guided samples still remain below the reference MS-COCO data, which reaches $40.52$ at the same sample size, but the guidance reduces the measured gap between generated and reference samples.

\subsection{Ablation Study Analysis}

Tables~\ref{tab:diversity_metrics_gap_20k_t5_long}, \ref{tab:diversity_metrics_gap_20k_poly3_long}, and \ref{tab:diversity_metrics_gap_20k_cosine_long} examine robustness to embedding and kernel choices. The absolute scale of conditional VNE and conditional RKE changes substantially across feature maps and kernels, as expected, but the qualitative conclusion is unchanged: the training reference has the highest conditional diversity, greedy decoding has the largest gap, and stochastic decoding reduces but does not eliminate the gap. This consistency indicates that the observed conditional diversity gap is not an artifact of a single embedding model or kernel family.

Table~\ref{tab:diversity_metrics_gap_20k_t5_long} repeats the LLM analysis with T5 embeddings. Even under this alternative text representation, all generated groups remain below the corresponding training reference. At $20$K samples, the conditional VNE gaps under ancestral sampling are $10.11$ for OLMo, $6.94$ for Pythia, and $5.63$ for GPT-Neo, while greedy decoding has much larger gaps. Tables~\ref{tab:diversity_metrics_gap_20k_poly3_long} and \ref{tab:diversity_metrics_gap_20k_cosine_long} show the same ranking under degree-3 polynomial and cosine kernels, respectively, although the numerical scale is compressed relative to the default Gaussian-kernel evaluation.

Table~\ref{tab:diversity_metrics_gap_20k_prefix3_long} studies the effect of increasing the prefix length to $4$ tokens with $2$ continuation tokens. The training reference still has higher conditional VNE and conditional RKE than all generated groups for all three models. The gaps are generally smaller than in the default $3$-token-prefix setting, which is consistent with stronger conditioning reducing the range of plausible short continuations. Nevertheless, the decoding-order pattern remains unchanged: greedy decoding is least diverse, followed by nucleus sampling and ancestral sampling.

Table~\ref{tab:token_length_ablation_full} reports the full prefix--continuation length ablation for conditional VNE. Increasing continuation length tends to increase the absolute conditional VNE for both training and generated continuations because the output space becomes larger. The average gap remains positive in every tested configuration, confirming that the conditional diversity gap persists across the evaluated tokenization choices. The gap is often larger for shorter prefixes and longer continuations, where the condition leaves more uncertainty about valid outputs and the model must cover a broader conditional continuation space.

\subsection{Chunk-Size Ablation for Long-Sequence Representations}
\label{app:chunk_size_ablation}

The long-sequence experiments in Section~\ref{sec:long_sequence_gap} represent each sequence as a concatenation of $2$-token segment embeddings. To examine the sensitivity of the reported gap to this choice, we vary the chunk size over $\{1,2,4\}$ tokens per segment on GPT-Neo. Since chunk size $=2$ corresponds to the configuration already reported in Table~\ref{tab:long_sequence_gap}, here we report the two remaining cases: Tables~\ref{tab:chunk_size_1} and \ref{tab:chunk_size_4} show the exponentiated conditional VNE for chunk sizes $1$ and $4$, respectively. Across all three chunk sizes, the training data consistently achieve higher conditional VNE than the generated continuations under every decoding strategy and every prefix--continuation configuration. The absolute scale of the conditional VNE changes with the chunk size, but the qualitative ordering training $>$ ancestral $>$ nucleus $>$ greedy and the positivity of the gap are preserved. This confirms that the observed long-sequence conditional diversity gap is not an artifact of the specific chunk size used in the main long-sequence analysis.

\begin{table}[t]
\centering
\caption{
Chunk-size ablation for long-sequence conditional VNE on GPT-Neo with chunk size $=1$.
We report the exponential of conditional VNE for the training reference and generated continuations under greedy decoding, nucleus sampling, and ancestral sampling, using prefix lengths $\{8,16,32\}$ and continuation lengths $\{16,32,64\}$.
}
\label{tab:chunk_size_1}
\begin{tabular}{llcccc}
\toprule
\textbf{Prefix} & \textbf{Continuation} & \textbf{Training} & \textbf{Greedy} & \textbf{Nucleus} & \textbf{Ancestral} \\
\midrule
$8$  & $16$ & $437.3$ & $337.3$ & $400.7$ & $410.2$ \\
$8$  & $32$ & $474.5$ & $351.8$ & $425.2$ & $436.9$ \\
$8$  & $64$ & $503.7$ & $365.3$ & $446.2$ & $462.3$ \\
$16$ & $16$ & $431.3$ & $380.3$ & $409.3$ & $415.4$ \\
$16$ & $32$ & $455.6$ & $389.7$ & $422.9$ & $432.9$ \\
$16$ & $64$ & $477.9$ & $400.8$ & $435.4$ & $447.5$ \\
$32$ & $16$ & $432.9$ & $406.6$ & $420.3$ & $424.9$ \\
$32$ & $32$ & $446.3$ & $409.3$ & $427.8$ & $433.7$ \\
$32$ & $64$ & $462.1$ & $413.6$ & $434.4$ & $442.3$ \\
\bottomrule
\end{tabular}
\end{table}

\begin{table}[t]
\centering
\caption{
Chunk-size ablation for long-sequence conditional VNE on GPT-Neo with chunk size $=4$.
We report the exponential of conditional VNE for the training reference and generated continuations under greedy decoding, nucleus sampling, and ancestral sampling, using prefix lengths $\{8,16,32\}$ and continuation lengths $\{16,32,64\}$.
}
\label{tab:chunk_size_4}
\begin{tabular}{llcccc}
\toprule
\textbf{Prefix} & \textbf{Continuation} & \textbf{Training} & \textbf{Greedy} & \textbf{Nucleus} & \textbf{Ancestral} \\
\midrule
$8$  & $16$ & $454.7$ & $382.5$ & $430.5$ & $433.4$ \\
$8$  & $32$ & $515.0$ & $415.2$ & $479.8$ & $482.0$ \\
$8$  & $64$ & $572.7$ & $435.4$ & $518.1$ & $525.0$ \\
$16$ & $16$ & $455.0$ & $428.2$ & $443.1$ & $443.1$ \\
$16$ & $32$ & $496.6$ & $454.4$ & $475.9$ & $477.6$ \\
$16$ & $64$ & $541.3$ & $474.1$ & $503.7$ & $507.4$ \\
$32$ & $16$ & $459.8$ & $450.8$ & $455.8$ & $456.5$ \\
$32$ & $32$ & $484.1$ & $468.0$ & $475.5$ & $475.4$ \\
$32$ & $64$ & $514.8$ & $483.0$ & $493.9$ & $495.1$ \\
\bottomrule
\end{tabular}
\end{table}

\subsection{Comparison with Temperature Scaling}
\label{app:temperature_baseline}

A natural question is whether the same diversity gain can be achieved by simply raising the sampling temperature. Theorem~\ref{thm:conditional_entropy_projection_principle} does not necessarily apply to temperature scaling, because changing the temperature need not produce the closest distribution satisfying the entropy constraint. To examine this baseline empirically, we repeated the OLMo experiment using ancestral sampling with different temperatures. Increasing the temperature raises conditional entropy but does not eliminate the gap within the evaluated range; even at a relatively high temperature, the conditional VNE remains below the training-data value. We then searched for a temperature that matches the training-data conditional entropy and found that approximately $T\approx 1.9$ achieves this target. However, matching entropy alone did not match the underlying data distribution.

Table~\ref{tab:temperature_vs_projection} reports two independent quality-oriented measures: precision~\cite{lebronnec2024exploring}, measuring generated probability mass within the estimated support of the training continuations, and external conditional negative log likelihood, computed using an independent language model conditioned on the same prefix. Although $T\approx 1.9$ approximately matches the reference entropy, it yields worse precision and worse external NLL than the training data. In contrast, the proposed projection achieves substantially higher precision and lower external NLL. These results support the distinction motivated by Theorem~\ref{thm:conditional_entropy_projection_principle}: simply matching entropy is insufficient, and the adjustment should be performed through the closest-distribution projection rather than an arbitrary entropy increase.

\begin{table}[t]
\centering
\caption{
Comparison of temperature scaling and the proposed entropy projection on OLMo.
Precision measures generated probability mass within the estimated support of the training continuations~\cite{lebronnec2024exploring}.
External conditional NLL is computed using an independent language model conditioned on the same prefix.
The temperature $T\approx 1.9$ is tuned to approximately match the training-data conditional entropy.
}
\label{tab:temperature_vs_projection}
\begin{tabular}{lcc}
\toprule
\textbf{Method} & \textbf{Precision} $\uparrow$ & \textbf{Ext. Cond. NLL} $\downarrow$ \\
\midrule
Training data & $1.000$ & $4.679$ \\
Temperature $T\approx 1.9$ & $0.930$ & $4.938$ \\
Entropy projection (ours) & $0.979$ & $3.713$ \\
\bottomrule
\end{tabular}
\end{table}

\subsection{Qualitative Probe of the Reweighting Behavior}
\label{app:depeaking_probe}

To inspect how the projection redistributes probability mass, we conduct a controlled probe. We take two prompts, ``A vibrant city in Northern America is \_\_\_'' and ``A renowned celebrity is \_\_\_'', sample $1000$ continuations for each, and reweight them using the conditional entropy projection. The per-output mass shift reveals a clean de-peaking pattern: weight leaves the over-represented dominant modes and moves to equally valid but under-represented alternatives. Table~\ref{tab:depeaking_probe} reports the before/after weights for the top dominant modes and selected up-weighted tail candidates. Quantitatively, the discrete entropy and the corresponding effective number of distinct continuations both increase: for the city probe, the discrete entropy rises from $2.587$ to $2.853$ and the effective number of distinct continuations rises from $13.29$ to $17.34$; for the celebrity probe, the entropy rises from $3.979$ to $4.398$ and the effective number rises from $53.46$ to $81.25$. These examples make the practical effect of the projection transparent: it reduces excessive concentration on dominant generated answers and increases the representation of plausible alternatives already present in the candidate pool, rather than promoting low-quality or implausible outputs.

\begin{table}[t]
\centering
\caption{
Qualitative probe of the entropy projection on two prompts.
We report the probability mass before and after projection for the top down-weighted dominant modes and selected up-weighted tail candidates.
``Before'' and ``After'' denote the mass assigned by the original uniformly weighted candidates and the reweighted candidates, respectively.
}
\label{tab:depeaking_probe}
\resizebox{\textwidth}{!}{
\begin{tabular}{l p{0.40\textwidth} p{0.40\textwidth}}
\toprule
\textbf{Prompt} & \textbf{Down-weighted (dominant modes): Before $\to$ After} & \textbf{Up-weighted (plausible tail): Before $\to$ After} \\
\midrule
``A vibrant city in Northern America is \_\_\_''
& Toronto: $0.206 \to 0.177$ \newline Vancouver: $0.198 \to 0.137$ \newline New York: $0.096 \to 0.068$
& Miami: $0.018 \to 0.046$ \newline Denver: $0.020 \to 0.035$ \newline Philadelphia: $0.010 \to 0.022$ \\
\midrule
``A renowned celebrity is \_\_\_''
& Oprah Winfrey: $0.134 \to 0.064$ \newline Madonna: $0.068 \to 0.036$ \newline Michael Jackson: $0.052 \to 0.020$ \newline Tom Cruise: $0.032 \to 0.020$
& Jackie Chan: $0.010 \to 0.019$ \newline Kate Middleton: $0.007 \to 0.013$ \newline Queen Victoria: $0.004 \to 0.015$ \newline J.~K.~Rowling: $0.002 \to 0.012$ \\
\bottomrule
\end{tabular}
}
\end{table}

\subsection{Lexical Audit of Conditional Diversity}
\label{app:lexical_audit}

To examine whether the observed conditional-entropy ordering also reflects in a directly human-interpretable form, we perform a word-level audit using the same OLMo setting as the main text with $20{,}000$ samples. We count exact, case-insensitive occurrences from three fixed and enumerable semantic lexicons: countries ($192$ terms), sports ($80$ terms), and occupations ($118$ terms). For each category, we normalize the term counts into a categorical distribution and report the number of distinct terms observed as well as the effective number of equally probable lexical choices $2^{H}$, where $H$ is the Shannon entropy of the normalized distribution in bits.

Table~\ref{tab:lexical_audit} reports the results. Importantly, all three independent lexical audits recover the same ordering as conditional VNE: training $>$ ancestral $>$ nucleus $>$ greedy. The effect is also concrete at the level of frequencies. For example, in country mentions, greedy decoding assigns $20.0\%$ of its mass to ``UK'' and $19.8\%$ to ``United States'', compared with only $8.4\%$ and $4.3\%$ in the training continuations. This indicates that the conditional VNE gap corresponds to a measurable shortfall in the semantic variety of generated continuations, not merely a change in the kernel-derived score.

\begin{table*}[t]
\centering
\caption{
Lexical audit of conditional diversity on OLMo at sample size $20$K.
For each semantic category, we report the number of distinct terms observed and the effective number of equally probable lexical choices $2^{H}$, where $H$ is the Shannon entropy of the normalized term-frequency distribution in bits.
The ordering across all three categories matches the conditional VNE ordering: training $>$ ancestral $>$ nucleus $>$ greedy.
}
\label{tab:lexical_audit}
\resizebox{\textwidth}{!}{
\begin{tabular}{lccccccc}
\toprule
\textbf{Group} & \textbf{Cond.VNE}
& \textbf{Countries: distinct} & \textbf{Countries: $2^{H}$}
& \textbf{Sports: distinct} & \textbf{Sports: $2^{H}$}
& \textbf{Occupations: distinct} & \textbf{Occupations: $2^{H}$} \\
\midrule
Training data & $414.2$ & $95$ & $54.4$ & $35$ & $22.4$ & $75$ & $49.9$ \\
Greedy         & $369.4$ & $82$ & $26.9$ & $20$ & $13.6$ & $66$ & $35.9$ \\
Nucleus        & $390.8$ & $82$ & $37.3$ & $25$ & $17.9$ & $69$ & $42.6$ \\
Ancestral      & $394.9$ & $83$ & $41.9$ & $24$ & $18.4$ & $68$ & $44.7$ \\
\bottomrule
\end{tabular}
}
\end{table*}

\subsection{Conditional Diversity Gaps in Larger Language Models}
\label{app:larger_models}

The main experiments evaluate approximately $1$B-parameter models. To examine whether the observed gap is restricted to smaller or earlier models, we additionally evaluate OLMo-3-7B and extend the Pythia experiments to $2.8$B and $6.9$B parameters under the same setting. Tables~\ref{tab:olmo3_7b_cvs} and \ref{tab:pythia_larger_cvs} report the exponentiated conditional VNE across sample sizes from $2.5$K to $20$K.

For OLMo-3-7B, the relative reductions at $20$K are approximately $33.0\%$, $16.2\%$, and $12.0\%$ under greedy, nucleus, and ancestral decoding, respectively. For Pythia-2.8B and Pythia-6.9B, the corresponding gaps also remain substantial. These results indicate that the conditional output-range gap is not restricted to the earlier approximately $1$B-parameter models and does not disappear merely by increasing model size.

\begin{table*}[t]
\centering
\caption{
Conditional diversity scores for OLMo-3-7B across sample sizes.
We report the exponential of conditional VNE for the training reference and generated continuations under greedy decoding, nucleus sampling, and ancestral sampling.
}
\label{tab:olmo3_7b_cvs}
\resizebox{\textwidth}{!}{
\renewcommand{\arraystretch}{1.2}
\begin{tabular}{lcccccccc}
\toprule
\textbf{Group} & \textbf{2.5K} & \textbf{5K} & \textbf{7.5K} & \textbf{10K} & \textbf{12.5K} & \textbf{15K} & \textbf{17.5K} & \textbf{20K} \\
\midrule
Training data & $86.92$  & $138.84$ & $183.04$ & $220.30$ & $254.52$ & $284.20$ & $312.61$ & $339.17$ \\
Greedy        & $71.92$  & $107.87$ & $135.95$ & $158.50$ & $178.47$ & $196.05$ & $212.69$ & $227.40$ \\
Nucleus       & $79.55$  & $125.16$ & $161.26$ & $191.71$ & $217.79$ & $241.40$ & $264.24$ & $284.58$ \\
Ancestral     & $81.80$  & $128.47$ & $166.06$ & $198.25$ & $226.70$ & $252.94$ & $276.80$ & $298.40$ \\
\bottomrule
\end{tabular}
\renewcommand{\arraystretch}{1.0}
}
\end{table*}

\begin{table*}[t]
\centering
\caption{
Conditional diversity scores for Pythia-2.8B and Pythia-6.9B across sample sizes.
We report the exponential of conditional VNE for the training reference and generated continuations under greedy decoding, nucleus sampling, and ancestral sampling.
}
\label{tab:pythia_larger_cvs}
\resizebox{\textwidth}{!}{
\renewcommand{\arraystretch}{1.2}
\begin{tabular}{llcccccccc}
\toprule
\textbf{Model} & \textbf{Group} & \textbf{2.5K} & \textbf{5K} & \textbf{7.5K} & \textbf{10K} & \textbf{12.5K} & \textbf{15K} & \textbf{17.5K} & \textbf{20K} \\
\midrule

\multirow{4}{*}{\rotatebox[origin=c]{90}{Pythia-2.8B}}
& Training data & $72.60$  & $114.80$ & $149.39$ & $177.89$ & $203.14$ & $225.89$ & $247.28$ & $265.87$ \\
& Greedy        & $58.49$  & $86.69$  & $108.65$ & $125.62$ & $140.05$ & $152.45$ & $163.21$ & $172.54$ \\
& Nucleus       & $66.07$  & $101.37$ & $131.11$ & $155.42$ & $176.17$ & $194.62$ & $211.69$ & $226.57$ \\
& Ancestral     & $68.98$  & $105.77$ & $136.45$ & $161.83$ & $182.24$ & $201.95$ & $220.13$ & $236.08$ \\
\midrule

\multirow{4}{*}{\rotatebox[origin=c]{90}{Pythia-6.9B}}
& Training data & $72.60$  & $114.80$ & $149.39$ & $177.89$ & $203.14$ & $225.89$ & $247.28$ & $265.87$ \\
& Greedy        & $58.18$  & $86.95$  & $109.39$ & $126.75$ & $141.22$ & $153.96$ & $164.73$ & $174.13$ \\
& Nucleus       & $67.94$  & $105.16$ & $135.18$ & $159.59$ & $179.85$ & $199.01$ & $215.33$ & $230.26$ \\
& Ancestral     & $68.79$  & $106.71$ & $138.51$ & $164.17$ & $185.90$ & $206.12$ & $224.19$ & $239.50$ \\
\bottomrule
\end{tabular}
\renewcommand{\arraystretch}{1.0}
}
\end{table*}

\begin{table*}[t]
\centering
\caption{
Ablation study with T5 embeddings for kernel-based conditional diversity metrics at sample size 20K.
The standard deviations of conditional VNE and conditional RKE are below $0.01$ across 5 independent runs.
The gap is computed as $\Delta = \mathrm{Metric}_{\mathrm{train}} - \mathrm{Metric}_{\mathrm{group}}$.
Positive gaps indicate lower diversity than the training data.
}
\label{tab:diversity_metrics_gap_20k_t5_long}

\renewcommand{\arraystretch}{1.0}
\begin{tabular*}{0.95\textwidth}{@{\extracolsep{\fill}}cllcc}
\toprule
\textbf{Model} & \textbf{Group} & \textbf{Metric} & \textbf{Value} & $\boldsymbol{\Delta}$ \\
\midrule

\multirow{16}{*}{\rotatebox[origin=c]{90}{\large\textbf{OLMo}}}
& \multirow{4}{*}{Training Data}
& Exp.Cond.VNE  & 49.74 & -- \\
& & Exp.Cond.RKE  & 4.25  & -- \\
& & Distinct-2 & 0.769 & -- \\
& & Distinct-3 & 0.964 & -- \\
\cmidrule(lr){2-5}
& \multirow{4}{*}{Greedy Decoding}
& Exp.Cond.VNE  & 28.22 & 21.52 \\
& & Exp.Cond.RKE  & 3.34  & 0.91 \\
& & Distinct-2 & 0.337 & 0.432 \\
& & Distinct-3 & 0.538 & 0.426 \\
\cmidrule(lr){2-5}
& \multirow{4}{*}{Nucleus Sampling}
& Exp.Cond.VNE  & 37.94 & 11.81 \\
& & Exp.Cond.RKE  & 3.74  & 0.51 \\
& & Distinct-2 & 0.588 & 0.181 \\
& & Distinct-3 & 0.889 & 0.075 \\
\cmidrule(lr){2-5}
& \multirow{4}{*}{Ancestral Sampling}
& Exp.Cond.VNE  & 39.64 & 10.11 \\
& & Exp.Cond.RKE  & 3.81  & 0.44 \\
& & Distinct-2 & 0.627 & 0.142 \\
& & Distinct-3 & 0.917 & 0.047 \\

\midrule

\multirow{16}{*}{\rotatebox[origin=c]{90}{\large\textbf{Pythia}}}
& \multirow{4}{*}{Training Data}
& Exp.Cond.VNE  & 43.06 & -- \\
& & Exp.Cond.RKE  & 4.15  & -- \\
& & Distinct-2 & 0.755 & -- \\
& & Distinct-3 & 0.930 & -- \\
\cmidrule(lr){2-5}
& \multirow{4}{*}{Greedy Decoding}
& Exp.Cond.VNE  & 26.36 & 16.70 \\
& & Exp.Cond.RKE  & 3.43  & 0.72 \\
& & Distinct-2 & 0.325 & 0.430 \\
& & Distinct-3 & 0.497 & 0.433 \\
\cmidrule(lr){2-5}
& \multirow{4}{*}{Nucleus Sampling}
& Exp.Cond.VNE  & 34.31 & 8.75 \\
& & Exp.Cond.RKE  & 3.76  & 0.39 \\
& & Distinct-2 & 0.597 & 0.158 \\
& & Distinct-3 & 0.866 & 0.064 \\
\cmidrule(lr){2-5}
& \multirow{4}{*}{Ancestral Sampling}
& Exp.Cond.VNE  & 36.12 & 6.94 \\
& & Exp.Cond.RKE  & 3.83  & 0.32 \\
& & Distinct-2 & 0.632 & 0.123 \\
& & Distinct-3 & 0.890 & 0.040 \\

\midrule

\multirow{16}{*}{\rotatebox[origin=c]{90}{\large\textbf{GPT-Neo}}}
& \multirow{4}{*}{Training Data}
& Exp.Cond.VNE  & 43.07 & -- \\
& & Exp.Cond.RKE  & 4.15  & -- \\
& & Distinct-2 & 0.755 & -- \\
& & Distinct-3 & 0.929 & -- \\
\cmidrule(lr){2-5}
& \multirow{4}{*}{Greedy Decoding}
& Exp.Cond.VNE  & 27.18 & 14.49 \\
& & Exp.Cond.RKE  & 3.51  & 0.64 \\
& & Distinct-2 & 0.344 & 0.411 \\
& & Distinct-3 & 0.510 & 0.419 \\
\cmidrule(lr){2-5}
& \multirow{4}{*}{Nucleus Sampling}
& Exp.Cond.VNE  & 34.21 & 7.46 \\
& & Exp.Cond.RKE  & 3.78  & 0.37 \\
& & Distinct-2 & 0.599 & 0.156 \\
& & Distinct-3 & 0.856 & 0.073 \\
\cmidrule(lr){2-5}
& \multirow{4}{*}{Ancestral Sampling}
& Exp.Cond.VNE  & 36.04 & 5.63 \\
& & Exp.Cond.RKE  & 3.85  & 0.30 \\
& & Distinct-2 & 0.638 & 0.117 \\
& & Distinct-3 & 0.882 & 0.047 \\

\bottomrule
\end{tabular*}
\renewcommand{\arraystretch}{1.0}

\end{table*}

\begin{table*}[t]
\centering
\caption{
Ablation study with a degree-3 polynomial kernel for kernel-based conditional diversity metrics at sample size 20K.
The standard deviations of conditional VNE and conditional RKE are below $0.01$ across 5 independent runs.
The gap is computed as $\Delta = \mathrm{Metric}_{\mathrm{train}} - \mathrm{Metric}_{\mathrm{group}}$.
Positive gaps indicate lower diversity than the training data.
}
\label{tab:diversity_metrics_gap_20k_poly3_long}

\renewcommand{\arraystretch}{1.0}
\begin{tabular*}{0.95\textwidth}{@{\extracolsep{\fill}}cllcc}
\toprule
\textbf{Model} & \textbf{Group} & \textbf{Metric} & \textbf{Value} & $\boldsymbol{\Delta}$ \\
\midrule

\multirow{16}{*}{\rotatebox[origin=c]{90}{\large\textbf{OLMo}}}
& \multirow{4}{*}{Training Data}
& Exp.Cond.VNE  & 8.71  & -- \\
& & Exp.Cond.RKE  & 2.54  & -- \\
& & Distinct-2 & 0.769 & -- \\
& & Distinct-3 & 0.964 & -- \\
\cmidrule(lr){2-5}
& \multirow{4}{*}{Greedy Decoding}
& Exp.Cond.VNE  & 6.28  & 2.43 \\
& & Exp.Cond.RKE  & 2.41  & 0.13 \\
& & Distinct-2 & 0.337 & 0.432 \\
& & Distinct-3 & 0.538 & 0.426 \\
\cmidrule(lr){2-5}
& \multirow{4}{*}{Nucleus Sampling}
& Exp.Cond.VNE  & 7.56  & 1.15 \\
& & Exp.Cond.RKE  & 2.46  & 0.08 \\
& & Distinct-2 & 0.588 & 0.181 \\
& & Distinct-3 & 0.889 & 0.075 \\
\cmidrule(lr){2-5}
& \multirow{4}{*}{Ancestral Sampling}
& Exp.Cond.VNE  & 7.68  & 1.03 \\
& & Exp.Cond.RKE  & 2.47  & 0.07 \\
& & Distinct-2 & 0.627 & 0.142 \\
& & Distinct-3 & 0.917 & 0.047 \\

\midrule

\multirow{16}{*}{\rotatebox[origin=c]{90}{\large\textbf{Pythia}}}
& \multirow{4}{*}{Training Data}
& Exp.Cond.VNE  & 7.64  & -- \\
& & Exp.Cond.RKE  & 2.50  & -- \\
& & Distinct-2 & 0.755 & -- \\
& & Distinct-3 & 0.930 & -- \\
\cmidrule(lr){2-5}
& \multirow{4}{*}{Greedy Decoding}
& Exp.Cond.VNE  & 5.82  & 1.82 \\
& & Exp.Cond.RKE  & 2.26  & 0.24 \\
& & Distinct-2 & 0.325 & 0.430 \\
& & Distinct-3 & 0.497 & 0.433 \\
\cmidrule(lr){2-5}
& \multirow{4}{*}{Nucleus Sampling}
& Exp.Cond.VNE  & 7.14  & 0.50 \\
& & Exp.Cond.RKE  & 2.36  & 0.14 \\
& & Distinct-2 & 0.597 & 0.158 \\
& & Distinct-3 & 0.866 & 0.064 \\
\cmidrule(lr){2-5}
& \multirow{4}{*}{Ancestral Sampling}
& Exp.Cond.VNE  & 7.21  & 0.43 \\
& & Exp.Cond.RKE  & 2.37  & 0.13 \\
& & Distinct-2 & 0.632 & 0.123 \\
& & Distinct-3 & 0.890 & 0.040 \\

\midrule

\multirow{16}{*}{\rotatebox[origin=c]{90}{\large\textbf{GPT-Neo}}}
& \multirow{4}{*}{Training Data}
& Exp.Cond.VNE  & 7.64  & -- \\
& & Exp.Cond.RKE  & 2.50  & -- \\
& & Distinct-2 & 0.755 & -- \\
& & Distinct-3 & 0.929 & -- \\
\cmidrule(lr){2-5}
& \multirow{4}{*}{Greedy Decoding}
& Exp.Cond.VNE  & 5.72  & 1.92 \\
& & Exp.Cond.RKE  & 2.24  & 0.26 \\
& & Distinct-2 & 0.344 & 0.411 \\
& & Distinct-3 & 0.510 & 0.419 \\
\cmidrule(lr){2-5}
& \multirow{4}{*}{Nucleus Sampling}
& Exp.Cond.VNE  & 7.01  & 0.63 \\
& & Exp.Cond.RKE  & 2.33  & 0.17 \\
& & Distinct-2 & 0.599 & 0.156 \\
& & Distinct-3 & 0.856 & 0.073 \\
\cmidrule(lr){2-5}
& \multirow{4}{*}{Ancestral Sampling}
& Exp.Cond.VNE  & 7.05  & 0.59 \\
& & Exp.Cond.RKE  & 2.34  & 0.16 \\
& & Distinct-2 & 0.638 & 0.117 \\
& & Distinct-3 & 0.882 & 0.047 \\

\bottomrule
\end{tabular*}
\renewcommand{\arraystretch}{1.0}

\end{table*}

\begin{table}[t]
\centering
\caption{
Ablation study with a cosine kernel for kernel-based conditional diversity metrics at sample size 20K.
The standard deviations of conditional VNE and conditional RKE are below $0.01$ across 5 independent runs.
The gap is computed as $\Delta = \mathrm{Metric}_{\mathrm{train}} - \mathrm{Metric}_{\mathrm{group}}$.
Positive gaps indicate lower diversity than the training data.
}
\label{tab:diversity_metrics_gap_20k_cosine_long}

\renewcommand{\arraystretch}{1.0}
\begin{tabular*}{0.95\textwidth}{@{\extracolsep{\fill}}cllcc}
\toprule
\textbf{Model} & \textbf{Group} & \textbf{Metric} & \textbf{Value} & $\boldsymbol{\Delta}$ \\
\midrule

\multirow{16}{*}{\rotatebox[origin=c]{90}{\large\textbf{OLMo}}}
& \multirow{4}{*}{Training Data}
& Exp.Cond.VNE  & 2.72  & -- \\
& & Exp.Cond.RKE  & 1.35  & -- \\
& & Distinct-2 & 0.769 & -- \\
& & Distinct-3 & 0.964 & -- \\
\cmidrule(lr){2-5}
& \multirow{4}{*}{Greedy Decoding}
& Exp.Cond.VNE  & 2.15  & 0.57 \\
& & Exp.Cond.RKE  & 1.24  & 0.11 \\
& & Distinct-2 & 0.337 & 0.432 \\
& & Distinct-3 & 0.538 & 0.426 \\
\cmidrule(lr){2-5}
& \multirow{4}{*}{Nucleus Sampling}
& Exp.Cond.VNE  & 2.38  & 0.34 \\
& & Exp.Cond.RKE  & 1.28  & 0.07 \\
& & Distinct-2 & 0.588 & 0.181 \\
& & Distinct-3 & 0.889 & 0.075 \\
\cmidrule(lr){2-5}
& \multirow{4}{*}{Ancestral Sampling}
& Exp.Cond.VNE  & 2.43  & 0.29 \\
& & Exp.Cond.RKE  & 1.29  & 0.06 \\
& & Distinct-2 & 0.627 & 0.142 \\
& & Distinct-3 & 0.917 & 0.047 \\

\midrule

\multirow{16}{*}{\rotatebox[origin=c]{90}{\large\textbf{Pythia}}}
& \multirow{4}{*}{Training Data}
& Exp.Cond.VNE  & 2.67  & -- \\
& & Exp.Cond.RKE  & 1.34  & -- \\
& & Distinct-2 & 0.755 & -- \\
& & Distinct-3 & 0.930 & -- \\
\cmidrule(lr){2-5}
& \multirow{4}{*}{Greedy Decoding}
& Exp.Cond.VNE  & 2.19  & 0.47 \\
& & Exp.Cond.RKE  & 1.25  & 0.09 \\
& & Distinct-2 & 0.325 & 0.430 \\
& & Distinct-3 & 0.497 & 0.433 \\
\cmidrule(lr){2-5}
& \multirow{4}{*}{Nucleus Sampling}
& Exp.Cond.VNE  & 2.43  & 0.24 \\
& & Exp.Cond.RKE  & 1.29  & 0.05 \\
& & Distinct-2 & 0.597 & 0.158 \\
& & Distinct-3 & 0.866 & 0.064 \\
\cmidrule(lr){2-5}
& \multirow{4}{*}{Ancestral Sampling}
& Exp.Cond.VNE  & 2.46  & 0.20 \\
& & Exp.Cond.RKE  & 1.30  & 0.04 \\
& & Distinct-2 & 0.632 & 0.123 \\
& & Distinct-3 & 0.890 & 0.040 \\

\midrule

\multirow{16}{*}{\rotatebox[origin=c]{90}{\large\textbf{GPT-Neo}}}
& \multirow{4}{*}{Training Data}
& Exp.Cond.VNE  & 2.67  & -- \\
& & Exp.Cond.RKE  & 1.34  & -- \\
& & Distinct-2 & 0.755 & -- \\
& & Distinct-3 & 0.929 & -- \\
\cmidrule(lr){2-5}
& \multirow{4}{*}{Greedy Decoding}
& Exp.Cond.VNE  & 2.27  & 0.41 \\
& & Exp.Cond.RKE  & 1.27  & 0.07 \\
& & Distinct-2 & 0.344 & 0.411 \\
& & Distinct-3 & 0.510 & 0.419 \\
\cmidrule(lr){2-5}
& \multirow{4}{*}{Nucleus Sampling}
& Exp.Cond.VNE  & 2.46  & 0.22 \\
& & Exp.Cond.RKE  & 1.29  & 0.05 \\
& & Distinct-2 & 0.599 & 0.156 \\
& & Distinct-3 & 0.856 & 0.073 \\
\cmidrule(lr){2-5}
& \multirow{4}{*}{Ancestral Sampling}
& Exp.Cond.VNE  & 2.50  & 0.17 \\
& & Exp.Cond.RKE  & 1.30  & 0.04 \\
& & Distinct-2 & 0.638 & 0.117 \\
& & Distinct-3 & 0.882 & 0.047 \\

\bottomrule
\end{tabular*}
\renewcommand{\arraystretch}{1.0}

\end{table}

\begin{table*}[t]
\centering
\caption{
Ablation study on prefix token length, with prefix token length set to 4 at sample size 20K.
The gap is computed as $\Delta = \mathrm{Metric}_{\mathrm{train}} - \mathrm{Metric}_{\mathrm{group}}$.
Positive gaps indicate lower diversity than the training data.
}
\label{tab:diversity_metrics_gap_20k_prefix3_long}

\renewcommand{\arraystretch}{1.0}
\begin{tabular*}{0.95\textwidth}{@{\extracolsep{\fill}}cllcc}
\toprule
\textbf{Model} & \textbf{Group} & \textbf{Metric} & \textbf{Value} & $\boldsymbol{\Delta}$ \\
\midrule

\multirow{16}{*}{\rotatebox[origin=c]{90}{\large\textbf{OLMo}}}
& \multirow{4}{*}{Training Data}
& Exp.Cond.VNE  & 310.48 & -- \\
& & Exp.Cond.RKE  & 100.06 & -- \\
& & Distinct-2 & 0.709 & -- \\
& & Distinct-3 & 0.947 & -- \\
\cmidrule(lr){2-5}
& \multirow{4}{*}{Greedy Decoding}
& Exp.Cond.VNE  & 245.72 & 64.77 \\
& & Exp.Cond.RKE  & 72.08  & 27.98 \\
& & Distinct-2 & 0.323 & 0.386 \\
& & Distinct-3 & 0.545 & 0.401 \\
\cmidrule(lr){2-5}
& \multirow{4}{*}{Nucleus Sampling}
& Exp.Cond.VNE  & 281.28 & 29.20 \\
& & Exp.Cond.RKE  & 83.67  & 16.39 \\
& & Distinct-2 & 0.534 & 0.175 \\
& & Distinct-3 & 0.858 & 0.089 \\
\cmidrule(lr){2-5}
& \multirow{4}{*}{Ancestral Sampling}
& Exp.Cond.VNE  & 287.63 & 22.85 \\
& & Exp.Cond.RKE  & 86.29  & 13.77 \\
& & Distinct-2 & 0.572 & 0.137 \\
& & Distinct-3 & 0.892 & 0.055 \\

\midrule

\multirow{16}{*}{\rotatebox[origin=c]{90}{\large\textbf{Pythia}}}
& \multirow{4}{*}{Training Data}
& Exp.Cond.VNE  & 240.47 & -- \\
& & Exp.Cond.RKE  & 84.79  & -- \\
& & Distinct-2 & 0.708 & -- \\
& & Distinct-3 & 0.918 & -- \\
\cmidrule(lr){2-5}
& \multirow{4}{*}{Greedy Decoding}
& Exp.Cond.VNE  & 193.58 & 46.88 \\
& & Exp.Cond.RKE  & 67.28  & 17.51 \\
& & Distinct-2 & 0.310 & 0.398 \\
& & Distinct-3 & 0.504 & 0.414 \\
\cmidrule(lr){2-5}
& \multirow{4}{*}{Nucleus Sampling}
& Exp.Cond.VNE  & 223.86 & 16.61 \\
& & Exp.Cond.RKE  & 76.50  & 8.29 \\
& & Distinct-2 & 0.537 & 0.171 \\
& & Distinct-3 & 0.834 & 0.084 \\
\cmidrule(lr){2-5}
& \multirow{4}{*}{Ancestral Sampling}
& Exp.Cond.VNE  & 229.11 & 11.36 \\
& & Exp.Cond.RKE  & 78.60  & 6.19 \\
& & Distinct-2 & 0.578 & 0.130 \\
& & Distinct-3 & 0.865 & 0.053 \\

\midrule

\multirow{16}{*}{\rotatebox[origin=c]{90}{\large\textbf{GPT-Neo}}}
& \multirow{4}{*}{Training Data}
& Exp.Cond.VNE  & 240.47 & -- \\
& & Exp.Cond.RKE  & 84.79  & -- \\
& & Distinct-2 & 0.709 & -- \\
& & Distinct-3 & 0.917 & -- \\
\cmidrule(lr){2-5}
& \multirow{4}{*}{Greedy Decoding}
& Exp.Cond.VNE  & 194.98 & 45.49 \\
& & Exp.Cond.RKE  & 70.01  & 14.78 \\
& & Distinct-2 & 0.330 & 0.379 \\
& & Distinct-3 & 0.510 & 0.407 \\
\cmidrule(lr){2-5}
& \multirow{4}{*}{Nucleus Sampling}
& Exp.Cond.VNE  & 223.74 & 16.73 \\
& & Exp.Cond.RKE  & 77.72  & 7.07 \\
& & Distinct-2 & 0.547 & 0.162 \\
& & Distinct-3 & 0.831 & 0.086 \\
\cmidrule(lr){2-5}
& \multirow{4}{*}{Ancestral Sampling}
& Exp.Cond.VNE  & 229.45 & 11.02 \\
& & Exp.Cond.RKE  & 79.69  & 5.10 \\
& & Distinct-2 & 0.582 & 0.127 \\
& & Distinct-3 & 0.863 & 0.055 \\

\bottomrule
\end{tabular*}
\renewcommand{\arraystretch}{1.0}

\end{table*}

\begin{table*}[t]
\centering
\caption{
Ablation study on sample size for diversity metrics across models and generation strategies at sample size 10K.
The gap is computed as $\Delta = \mathrm{Metric}_{\mathrm{train}} - \mathrm{Metric}_{\mathrm{group}}$.
Positive gaps indicate lower diversity than the training data.
}
\label{tab:diversity_metrics_gap_10k_long}

\renewcommand{\arraystretch}{1.0}
\begin{tabular*}{0.95\textwidth}{@{\extracolsep{\fill}}cllcc}
\toprule
\textbf{Model} & \textbf{Group} & \textbf{Metric} & \textbf{Value} & $\boldsymbol{\Delta}$ \\
\midrule

\multirow{16}{*}{\rotatebox[origin=c]{90}{\large\textbf{OLMo}}}
& \multirow{4}{*}{Training Data}
& Exp.Cond.VNE  & 221.31 & -- \\
& & Exp.Cond.RKE  & 67.33  & -- \\
& & Distinct-2 & 0.769  & -- \\
& & Distinct-3 & 0.964  & -- \\
\cmidrule(lr){2-5}
& \multirow{4}{*}{Greedy Decoding}
& Exp.Cond.VNE  & 154.97 & 66.34 \\
& & Exp.Cond.RKE  & 46.19  & 21.14 \\
& & Distinct-2 & 0.337  & 0.432 \\
& & Distinct-3 & 0.538  & 0.426 \\
\cmidrule(lr){2-5}
& \multirow{4}{*}{Nucleus Sampling}
& Exp.Cond.VNE  & 193.58 & 27.73 \\
& & Exp.Cond.RKE  & 55.44  & 11.89 \\
& & Distinct-2 & 0.588  & 0.181 \\
& & Distinct-3 & 0.889  & 0.075 \\
\cmidrule(lr){2-5}
& \multirow{4}{*}{Ancestral Sampling}
& Exp.Cond.VNE  & 198.51 & 22.80 \\
& & Exp.Cond.RKE  & 56.92  & 10.41 \\
& & Distinct-2 & 0.627  & 0.142 \\
& & Distinct-3 & 0.917  & 0.047 \\

\midrule

\multirow{16}{*}{\rotatebox[origin=c]{90}{\large\textbf{Pythia}}}
& \multirow{4}{*}{Training Data}
& Exp.Cond.VNE  & 176.51 & -- \\
& & Exp.Cond.RKE  & 52.89  & -- \\
& & Distinct-2 & 0.755  & -- \\
& & Distinct-3 & 0.930  & -- \\
\cmidrule(lr){2-5}
& \multirow{4}{*}{Greedy Decoding}
& Exp.Cond.VNE  & 127.46 & 49.05 \\
& & Exp.Cond.RKE  & 40.37  & 12.52 \\
& & Distinct-2 & 0.325  & 0.430 \\
& & Distinct-3 & 0.497  & 0.433 \\
\cmidrule(lr){2-5}
& \multirow{4}{*}{Nucleus Sampling}
& Exp.Cond.VNE  & 159.36 & 17.15 \\
& & Exp.Cond.RKE  & 46.84  & 6.05 \\
& & Distinct-2 & 0.597  & 0.158 \\
& & Distinct-3 & 0.866  & 0.064 \\
\cmidrule(lr){2-5}
& \multirow{4}{*}{Ancestral Sampling}
& Exp.Cond.VNE  & 165.10 & 11.41 \\
& & Exp.Cond.RKE  & 48.27  & 4.62 \\
& & Distinct-2 & 0.632  & 0.123 \\
& & Distinct-3 & 0.890  & 0.040 \\

\midrule

\multirow{16}{*}{\rotatebox[origin=c]{90}{\large\textbf{GPT-Neo}}}
& \multirow{4}{*}{Training Data}
& Exp.Cond.VNE  & 175.98 & -- \\
& & Exp.Cond.RKE  & 52.60  & -- \\
& & Distinct-2 & 0.755  & -- \\
& & Distinct-3 & 0.929  & -- \\
\cmidrule(lr){2-5}
& \multirow{4}{*}{Greedy Decoding}
& Exp.Cond.VNE  & 130.11 & 45.87 \\
& & Exp.Cond.RKE  & 43.03  & 9.57 \\
& & Distinct-2 & 0.344  & 0.411 \\
& & Distinct-3 & 0.510  & 0.419 \\
\cmidrule(lr){2-5}
& \multirow{4}{*}{Nucleus Sampling}
& Exp.Cond.VNE  & 159.35 & 16.63 \\
& & Exp.Cond.RKE  & 47.93  & 4.67 \\
& & Distinct-2 & 0.599  & 0.156 \\
& & Distinct-3 & 0.856  & 0.073 \\
\cmidrule(lr){2-5}
& \multirow{4}{*}{Ancestral Sampling}
& Exp.Cond.VNE  & 163.39 & 12.59 \\
& & Exp.Cond.RKE  & 48.71  & 3.89 \\
& & Distinct-2 & 0.638  & 0.117 \\
& & Distinct-3 & 0.882  & 0.047 \\

\bottomrule
\end{tabular*}
\renewcommand{\arraystretch}{1.0}

\end{table*}

\begin{table*}[t]
\centering
\caption{
Ablation study on sample size for diversity metrics across models and generation strategies at sample size 20K.
The gap is computed as $\Delta = \mathrm{Metric}_{\mathrm{train}} - \mathrm{Metric}_{\mathrm{group}}$.
Positive gaps indicate lower diversity than the training data.
}
\label{tab:diversity_metrics_gap_20k_full_appendix}

\renewcommand{\arraystretch}{1.0}
\begin{tabular*}{0.95\textwidth}{@{\extracolsep{\fill}}cllcc}
\toprule
\textbf{Model} & \textbf{Group} & \textbf{Metric} & \textbf{Value} & $\boldsymbol{\Delta}$ \\
\midrule

\multirow{16}{*}{\rotatebox[origin=c]{90}{\large\textbf{OLMo}}}
& \multirow{4}{*}{Training Data}
& Exp.Cond.VNE  & 338.58 & -- \\
& & Exp.Cond.RKE  & 67.43  & -- \\
& & Distinct-2 & 0.709  & -- \\
& & Distinct-3 & 0.946  & -- \\
\cmidrule(lr){2-5}
& \multirow{4}{*}{Greedy Decoding}
& Exp.Cond.VNE  & 218.88 & 119.70 \\
& & Exp.Cond.RKE  & 46.17  & 21.26 \\
& & Distinct-2 & 0.276  & 0.433 \\
& & Distinct-3 & 0.465  & 0.481 \\
\cmidrule(lr){2-5}
& \multirow{4}{*}{Nucleus Sampling}
& Exp.Cond.VNE  & 287.31 & 51.27 \\
& & Exp.Cond.RKE  & 55.08  & 12.35 \\
& & Distinct-2 & 0.516  & 0.193 \\
& & Distinct-3 & 0.847  & 0.099 \\
\cmidrule(lr){2-5}
& \multirow{4}{*}{Ancestral Sampling}
& Exp.Cond.VNE  & 297.31 & 41.27 \\
& & Exp.Cond.RKE  & 56.81  & 10.62 \\
& & Distinct-2 & 0.558  & 0.151 \\
& & Distinct-3 & 0.883  & 0.063 \\

\midrule

\multirow{16}{*}{\rotatebox[origin=c]{90}{\large\textbf{Pythia}}}
& \multirow{4}{*}{Training Data}
& Exp.Cond.VNE  & 262.55 & -- \\
& & Exp.Cond.RKE  & 53.44  & -- \\
& & Distinct-2 & 0.701  & -- \\
& & Distinct-3 & 0.913  & -- \\
\cmidrule(lr){2-5}
& \multirow{4}{*}{Greedy Decoding}
& Exp.Cond.VNE  & 176.39 & 86.16 \\
& & Exp.Cond.RKE  & 40.98  & 12.46 \\
& & Distinct-2 & 0.266  & 0.435 \\
& & Distinct-3 & 0.423  & 0.490 \\
\cmidrule(lr){2-5}
& \multirow{4}{*}{Nucleus Sampling}
& Exp.Cond.VNE  & 232.24 & 30.31 \\
& & Exp.Cond.RKE  & 47.51  & 5.93 \\
& & Distinct-2 & 0.525  & 0.176 \\
& & Distinct-3 & 0.826  & 0.087 \\
\cmidrule(lr){2-5}
& \multirow{4}{*}{Ancestral Sampling}
& Exp.Cond.VNE  & 241.79 & 20.76 \\
& & Exp.Cond.RKE  & 49.00  & 4.44 \\
& & Distinct-2 & 0.564  & 0.137 \\
& & Distinct-3 & 0.857  & 0.056 \\

\midrule

\multirow{16}{*}{\rotatebox[origin=c]{90}{\large\textbf{GPT-Neo}}}
& \multirow{4}{*}{Training Data}
& Exp.Cond.VNE  & 260.01 & -- \\
& & Exp.Cond.RKE  & 52.98  & -- \\
& & Distinct-2 & 0.701  & -- \\
& & Distinct-3 & 0.912  & -- \\
\cmidrule(lr){2-5}
& \multirow{4}{*}{Greedy Decoding}
& Exp.Cond.VNE  & 179.36 & 80.65 \\
& & Exp.Cond.RKE  & 43.34  & 9.64 \\
& & Distinct-2 & 0.284  & 0.417 \\
& & Distinct-3 & 0.437  & 0.475 \\
\cmidrule(lr){2-5}
& \multirow{4}{*}{Nucleus Sampling}
& Exp.Cond.VNE  & 229.62 & 30.39 \\
& & Exp.Cond.RKE  & 48.48  & 4.50 \\
& & Distinct-2 & 0.526  & 0.175 \\
& & Distinct-3 & 0.812  & 0.100 \\
\cmidrule(lr){2-5}
& \multirow{4}{*}{Ancestral Sampling}
& Exp.Cond.VNE  & 238.50 & 21.51 \\
& & Exp.Cond.RKE  & 49.59  & 3.39 \\
& & Distinct-2 & 0.570  & 0.131 \\
& & Distinct-3 & 0.848  & 0.064 \\

\bottomrule
\end{tabular*}
\renewcommand{\arraystretch}{1.0}

\end{table*}

\begin{table*}[t]
\centering
\caption{
Full prefix and continuation length ablation for LLM conditional VNE.
We report the exponential of conditional VNE for the training reference and generated continuations under greedy decoding, nucleus sampling, and ancestral sampling.
The average gap is computed across the three decoding strategies.
}
\label{tab:token_length_ablation_full}
\resizebox{\textwidth}{!}{
\begin{tabular}{lccccccc}
\toprule
\textbf{Model}
& \textbf{Prefix}
& \textbf{Continuation}
& \textbf{Training}
& \textbf{Greedy}
& \textbf{Nucleus}
& \textbf{Ancestral}
& \textbf{Avg. Gap} \\
\midrule
\multirow{6}{*}{OLMo}
& $3$ & $2$ & 338.58 & 218.88 & 287.31 & 297.30 & 70.75 \\
& $3$ & $3$ & 617.63 & 404.37 & 564.52 & 581.67 & 100.78 \\
& $3$ & $4$ & 741.13 & 517.42 & 716.87 & 726.35 & 87.58 \\
& $4$ & $2$ & 310.48 & 245.72 & 281.27 & 287.63 & 38.94 \\
& $4$ & $3$ & 470.69 & 388.95 & 450.11 & 455.52 & 39.16 \\
& $4$ & $4$ & 532.24 & 463.86 & 525.02 & 529.51 & 26.11 \\
\midrule
\multirow{6}{*}{Pythia}
& $3$ & $2$ & 262.55 & 176.39 & 232.24 & 241.79 & 45.74 \\
& $3$ & $3$ & 492.00 & 301.65 & 439.49 & 456.91 & 92.65 \\
& $3$ & $4$ & 604.57 & 397.22 & 575.02 & 589.29 & 84.06 \\
& $4$ & $2$ & 240.46 & 193.58 & 223.85 & 229.10 & 24.95 \\
& $4$ & $3$ & 370.43 & 285.74 & 345.38 & 352.88 & 42.43 \\
& $4$ & $4$ & 423.95 & 342.08 & 408.47 & 414.80 & 35.50 \\
\midrule
\multirow{6}{*}{GPT-Neo}
& $3$ & $2$ & 262.55 & 179.36 & 229.62 & 238.49 & 46.73 \\
& $3$ & $3$ & 491.53 & 301.45 & 432.98 & 452.09 & 96.02 \\
& $3$ & $4$ & 607.21 & 391.38 & 561.90 & 579.97 & 96.13 \\
& $4$ & $2$ & 240.47 & 194.97 & 223.74 & 229.45 & 24.42 \\
& $4$ & $3$ & 371.87 & 284.75 & 341.55 & 352.06 & 45.75 \\
& $4$ & $4$ & 426.57 & 339.42 & 404.99 & 412.19 & 41.04 \\
\bottomrule
\end{tabular}
}
\end{table*}

\begin{figure}[t]
\begin{center}
\centerline{\includegraphics[width=\textwidth]{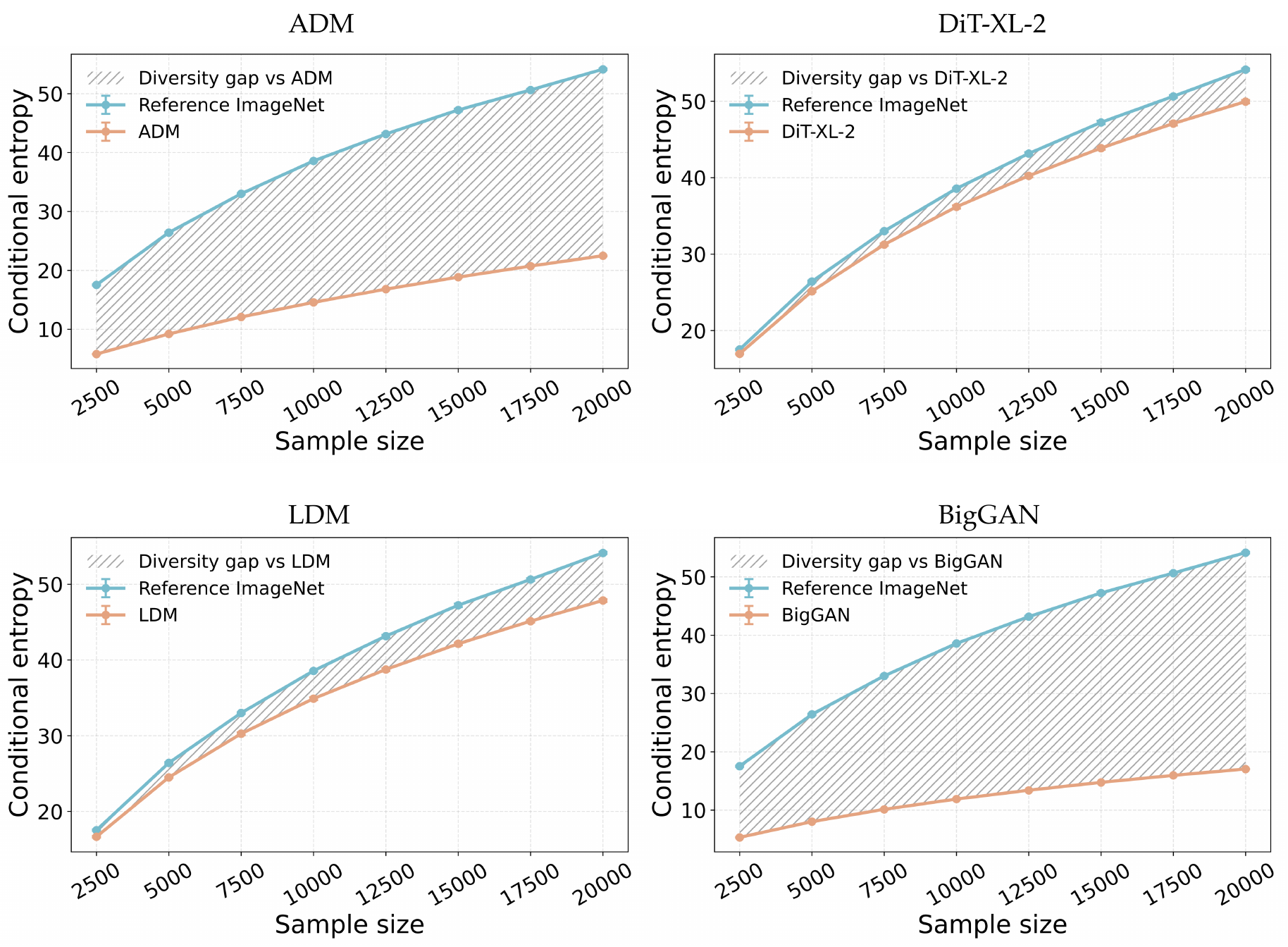}}
\caption{
Conditional diversity gap on class-conditioned ImageNet generation. 
We plot conditional VNE as a function of sample size for the ImageNet training-data reference and generated samples from ADM, DiT-XL-2, LDM, and BigGAN under matched class labels. 
Across all evaluated sample sizes, the reference ImageNet data achieves higher conditional VNE than the generated samples, and the gap generally increases as the sample size grows.
}
\label{fig:imagenet_gap}
\end{center}
\end{figure}

\begin{figure}[t]
\begin{center}
\centerline{\includegraphics[width=\textwidth]{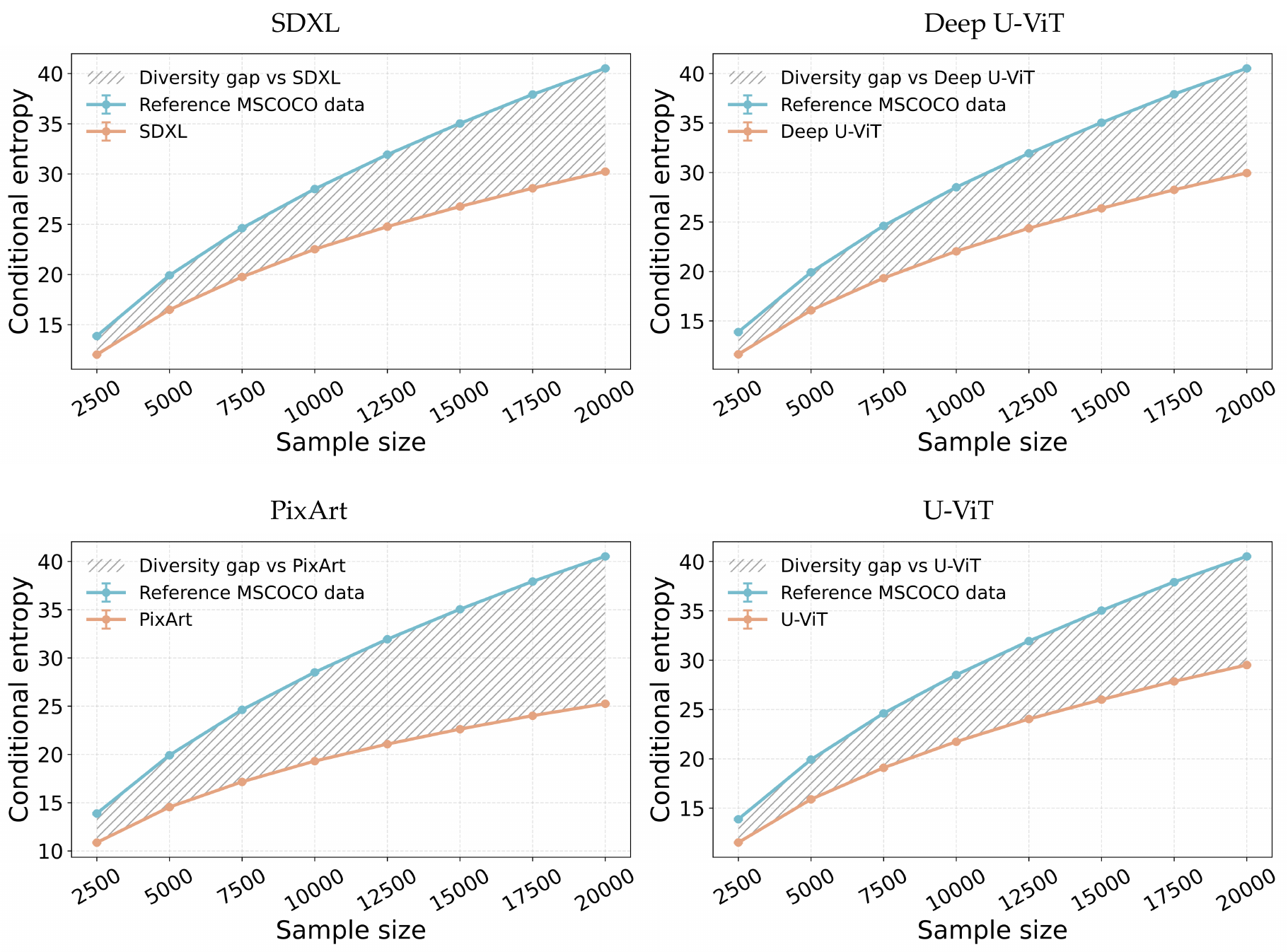}}
\caption{
Conditional diversity gap on text-conditioned MS-COCO generation. 
We plot conditional VNE as a function of sample size for the MS-COCO real-data reference and generated samples from SDXL, PixArt, U-ViT, and Deep U-ViT under matched text captions. 
Across all evaluated sample sizes, the real-data reference achieves higher conditional VNE than the generated samples, and the gap generally becomes larger with increasing sample size.
}
\label{fig:coco_gap}
\end{center}
\end{figure}

\begin{table*}[t]
\centering
\caption{
Conditional diversity scores for class-conditioned ImageNet generation across sample sizes.
Rows correspond to fixed sample sizes, and columns correspond to the ImageNet reference data and evaluated generative models.
We report the exponential of conditional VNE as mean $\pm$ standard deviation over 5 independent runs.
}
\label{tab:imagenet_cvs_full_transposed}
\resizebox{\textwidth}{!}{
\renewcommand{\arraystretch}{1.5}
\begin{tabular}{lccccc}
\toprule
\textbf{Sample Size}
& \rotatebox[origin=c]{60}{\textbf{ImageNet}}
& \rotatebox[origin=c]{60}{\textbf{LDM}}
& \rotatebox[origin=c]{60}{\textbf{ADM}}
& \rotatebox[origin=c]{60}{\textbf{BigGAN}}
& \rotatebox[origin=c]{60}{\textbf{DiT-XL-2}} \\
\midrule
\textbf{2.5K}
& $17.53 \pm 0.09$
& $16.66 \pm 0.11$
& $5.78 \pm 0.02$
& $5.32 \pm 0.03$
& $16.96 \pm 0.06$ \\
\textbf{5K}
& $26.41 \pm 0.12$
& $24.50 \pm 0.08$
& $9.20 \pm 0.04$
& $8.02 \pm 0.05$
& $25.14 \pm 0.10$ \\
\textbf{7.5K}
& $33.01 \pm 0.10$
& $30.27 \pm 0.19$
& $12.07 \pm 0.03$
& $10.14 \pm 0.06$
& $31.25 \pm 0.17$ \\
\textbf{10K}
& $38.57 \pm 0.08$
& $34.90 \pm 0.19$
& $14.56 \pm 0.04$
& $11.89 \pm 0.05$
& $36.18 \pm 0.22$ \\
\textbf{12.5K}
& $43.15 \pm 0.18$
& $38.75 \pm 0.18$
& $16.80 \pm 0.04$
& $13.40 \pm 0.05$
& $40.24 \pm 0.20$ \\
\textbf{15K}
& $47.23 \pm 0.23$
& $42.14 \pm 0.17$
& $18.84 \pm 0.06$
& $14.75 \pm 0.05$
& $43.86 \pm 0.19$ \\
\textbf{17.5K}
& $50.63 \pm 0.17$
& $45.13 \pm 0.17$
& $20.72 \pm 0.04$
& $15.95 \pm 0.04$
& $47.08 \pm 0.24$ \\
\textbf{20K}
& $54.15 \pm 0.14$
& $47.86 \pm 0.17$
& $22.47 \pm 0.03$
& $17.04 \pm 0.02$
& $49.95 \pm 0.21$ \\
\bottomrule
\end{tabular}
\renewcommand{\arraystretch}{1.0}
}
\end{table*}

\begin{table*}[t]
\centering
\caption{
Conditional diversity scores for text-conditioned MS-COCO generation across sample sizes.
Rows correspond to fixed sample sizes, and columns correspond to the MS-COCO reference data and evaluated text-to-image generative models.
We report the exponential of conditional VNE as mean $\pm$ standard deviation over independent runs.
}
\label{tab:mscoco_cvs_full_transposed}
\resizebox{\textwidth}{!}{
\renewcommand{\arraystretch}{1.5}
\begin{tabular}{lcccccc}
\toprule
\textbf{Sample Size}
& \rotatebox[origin=c]{60}{\textbf{MS-COCO}}
& \rotatebox[origin=c]{60}{\textbf{U-ViT Deep S/2}}
& \rotatebox[origin=c]{60}{\textbf{U-ViT S/2}}
& \rotatebox[origin=c]{60}{\textbf{SDXL}}
& \rotatebox[origin=c]{60}{\textbf{PixArt-$\alpha$}} \\
\midrule
\textbf{2.5K}
& $13.88 \pm 0.07$
& $11.63 \pm 0.06$
& $11.51 \pm 0.04$
& $12.02 \pm 0.02$
& $10.87 \pm 0.04$ \\
\textbf{5K}
& $19.91 \pm 0.04$
& $16.07 \pm 0.06$
& $15.89 \pm 0.09$
& $16.49 \pm 0.03$
& $14.54 \pm 0.05$ \\
\textbf{7.5K}
& $24.62 \pm 0.05$
& $19.33 \pm 0.07$
& $19.09 \pm 0.09$
& $19.75 \pm 0.14$
& $17.15 \pm 0.03$ \\
\textbf{10K}
& $28.51 \pm 0.06$
& $22.03 \pm 0.05$
& $21.73 \pm 0.05$
& $22.51 \pm 0.09$
& $19.30 \pm 0.01$ \\
\textbf{12.5K}
& $31.93 \pm 0.06$
& $24.36 \pm 0.05$
& $24.03 \pm 0.06$
& $24.76 \pm 0.07$
& $21.06 \pm 0.03$ \\
\textbf{15K}
& $35.03 \pm 0.03$
& $26.38 \pm 0.06$
& $26.00 \pm 0.06$
& $26.76 \pm 0.07$
& $22.62 \pm 0.02$ \\
\textbf{17.5K}
& $37.92 \pm 0.04$
& $28.24 \pm 0.05$
& $27.84 \pm 0.05$
& $28.58 \pm 0.04$
& $24.01 \pm 0.00$ \\
\textbf{20K}
& $40.52 \pm 0.03$
& $29.94 \pm 0.03$
& $29.50 \pm 0.04$
& $30.25 \pm 0.04$
& $25.25 \pm 0.01$ \\
\bottomrule
\end{tabular}
\renewcommand{\arraystretch}{1.0}
}
\end{table*}
\end{appendices}

\end{document}